\documentclass[twoside,11pt]{article}
\usepackage[preprint]{jmlr2e}
\usepackage{bbm}
\usepackage{amsmath}
\usepackage{amssymb}
\usepackage{comment}
\usepackage{graphicx}
\usepackage{cancel}
\usepackage{mleftright}\mleftright
\usepackage{enumerate}
\usepackage[symbol]{footmisc}
\renewcommand{\thefootnote}{\fnsymbol{footnote}}
\usepackage{wrapfig}
\usepackage{caption}
\usepackage{dsfont}
\usepackage{ stmaryrd }
\usepackage{setspace}
\usepackage{times}
\usepackage{subfigure} 
\usepackage[normalem]{ulem}
\usepackage{dsfont}
\usepackage{soul}
\numberwithin{equation}{section}

\usepackage{enumitem}

\usepackage{cleveref}
\firstpageno{1}
\usepackage{microtype}
\usepackage{booktabs} %
\usepackage{soul}

\usepackage{amsfonts}
\usepackage{color,xcolor}
\usepackage{url}
\usepackage{multirow}
\usepackage{mathtools}
\usepackage{silence}
\ActivateWarningFilters[pdftoc]
\ActivateWarningFilters[xclr]

\newcommand{\rset}{\mathbb{R}}

\newcommand{\sas}{{\cal S}\alpha{\cal S}}

\newcommand{\Rcal}{\mathcal{R}}

\newcommand{\Zcal}{\mathcal{Z}}
\newcommand{\veps}{\varepsilon}

\newtheorem{lem}{Lemma}

\newtheorem{cor}{Corollary}

\newtheorem{assumption}{H}

\Crefname{assumption}{\textnormal{\textbf{H}}\hspace{-3pt}}{\textnormal{\textbf{H}}\hspace{-3pt}}
\crefname{assumption}{\textnormal{\textbf{H}}}{\textnormal{\textbf{H}}}

\newcommand{\vc}[1]{\mathbf{#1}} 

\usepackage{xspace}

\Crefname{assumption}{\textnormal{\textbf{H}}\hspace{-3pt}}{\textnormal{\textbf{H}}\hspace{-3pt}}
\crefname{assumption}{\textnormal{\textbf{H}}}{\textnormal{\textbf{H}}}

\usepackage{appendix}
\usepackage{hyperref}
\hypersetup{bookmarksdepth=2}

\jmlrheading{1}{2021}{1-26}{02/21}{--/--}{tbd}{Barsbey et al.}
\ShortHeadings{Heavy Tails and Compressibility of Overparametrized Networks}{Heavy Tails and Compressibility of Overparametrized Networks}

\begin{document}

\title{Heavy Tails in SGD and Compressibility \\ of Overparametrized Neural Networks}

\setlength{\parskip}{\baselineskip}%
\setlength{\parindent}{0pt}%

\author{\name Melih Barsbey* \email melih.barsbey@boun.edu.tr \\
      \addr Department of Computer Engineering, Bo\u{g}azi\c{c}i University, Istanbul, Turkey
      \AND
\name Milad Sefidgaran* \email milad.sefidgaran@telecom-paris.fr \\
      \addr LTCI, Télécom Paris, Institut Polytechnique de Paris, France
      \AND
      \name Murat A. Erdogdu \email  erdogdu@cs.toronto.edu \\
      \addr Department of Computer Science and Department of Statistical Sciences at University of Toronto, and Vector
Institute, Toronto, Canada
            \AND
\name Ga\"{e}l Richard \email gael.richard@telecom-paris.fr \\
      \addr LTCI, Télécom Paris, Institut Polytechnique de Paris, France
      \AND
\name Umut \c{S}im\c{s}ekli \email umut.simsekli@inria.fr \\
      \addr INRIA - D\'{e}partement d’Informatique de l’\'{E}cole Normale Sup\'{e}rieure, PSL Research University, Paris, France
      }

 \footnotetext[1]{Equal contribution.}
\maketitle
\renewcommand{\thefootnote}{\arabic{footnote}}
\begin{abstract}
Neural network compression techniques have become increasingly popular as they can drastically reduce the storage and computation requirements for very large networks. Recent empirical studies have illustrated that even simple pruning strategies can be surprisingly effective, and several theoretical studies have shown that compressible networks (in specific senses) should achieve a low generalization error. Yet, a theoretical characterization of the underlying cause that makes the networks amenable to such simple compression schemes is still missing. In this study, we address this fundamental question and reveal that the dynamics of the training algorithm has a key role in obtaining such compressible networks. Focusing our attention on stochastic gradient descent (SGD), our main contribution is to link compressibility to two recently established properties of SGD: (i) as the network size goes to infinity, the system can converge to a \emph{mean-field} limit, where the network weights behave independently, (ii) for a large step-size/batch-size ratio, the SGD iterates can converge to a \emph{heavy-tailed} stationary distribution. In the case where these two phenomena occur \emph{simultaneously}, we prove that the networks are guaranteed to be `$\ell_p$-compressible', and the compression errors of different pruning techniques (magnitude, singular value, or node pruning) become arbitrarily small as the network size increases. We further prove generalization bounds adapted to our theoretical framework, which indeed confirm that the generalization error will be lower for more compressible networks. Our theory and numerical study on various neural networks show that large step-size/batch-size ratios introduce heavy-tails, which, in combination with overparametrization, result in compressibility.
\end{abstract}

\setcounter{tocdepth}{1}

\section{Introduction}
\label{sec:intro}
With the increasing model sizes in deep learning and with its increasing use in low-resource environments, network compression is becoming ever more important. Among many network compression techniques, network pruning has been arguably the most commonly used method \cite{oneillOverviewNeuralNetwork2020}, and it is rising in popularity and success \cite{blalock2020state}. Though various pruning methods are successfully used in practice and their theoretical implications in terms of generalization are increasingly apparent \cite{arora2018stronger}, a thorough understanding of \emph{why} and \emph{when} neural networks are compressible is lacking.

A common conclusion in pruning research is that overparametrized networks can be greatly compressed by pruning with little to no cost at generalization, including with simple schemes such as magnitude pruning \cite{blalock2020state, oneillOverviewNeuralNetwork2020}. For example, research on \textit{iterative magnitude pruning}  \cite{frankle2018the} demonstrated the possibility of compressing trained deep learning models by iteratively eliciting a much sparser substructure. While it is known that the choice of training hyperparameters such as learning rate affects the performance of such pruning strategies \cite{frankle_linear_2020, hattie_zhou_deconstructing_2020, renda_comparing_2020}, usually such observations are low in granularity and almost never theoretically motivated. Overall, the field lacks a framework to understand why or when a pruning method should be useful \cite{oneillOverviewNeuralNetwork2020}.

Another strand of research that highlights the importance of understanding network compressibility includes various studies \cite{arora2018stronger,suzuki2018spectral,Suzuki2020Compression,hsu2021generalization,kuhn2021robustness} that presented generalization bounds and/or empirical evidence that imply that the more compressible a network is, the more likely it is to generalize well. The aforementioned bounds are particularly interesting since classical generalization bounds increase with the dimension and hence become irrelevant in high dimensional deep learning settings, and fall short of explaining the generalization behavior of overparametrized neural networks. These results again illustrate the importance of understanding the conditions that give rise to compressibility given their implications regarding generalization. 

In this paper, we develop a theoretical framework to address (i) \emph{why and when modern neural networks can be amenable to very simple pruning strategies} and (ii) \emph{how this relates to generalization}. Our theory lies in the intersection of two recent discoveries regarding deep neural networks trained with the stochastic gradient descent (SGD) algorithm. The first one is the emergence of \emph{heavy-tailed} stationary distributions, which appear when the networks are trained with large learning rates and/or small batch-sizes \cite{hodgkinson2020multiplicative, gurbuzbalaban2020heavy}. The second one is the \emph{propagation of chaos} phenomenon, which indicates that, as the network size goes to infinity, the network weights behave \emph{independently} \cite{mei2018mean,sirignano2020mean,de2020quantitative}.

We show that, under the setting described above, fully connected neural networks will be provably compressible in a precise sense, and the compression errors of (i) unstructured global or layer-wise magnitude pruning, (ii) pruning based on singular values of the weight matrices, (iii) and node pruning can be made arbitrarily small for any compression ratio as the dimension increases. Our formulation of network compressibility in terms of `$\ell_p$-compressibility' enables us to access results from \textit{compressed sensing} literature \cite{gribonval2012compressible, amini2011compressibility} to be used in neural network analysis. Moreover, we prove generalization bounds adapted to our theoretical framework that agree with existing compression-based generalization bounds \cite{arora2018stronger,suzuki2018spectral,Suzuki2020Compression} and confirm that compressibility implies improved generalization. We conduct experiments on fully connected and convolutional neural networks and show that the results are in strong accordance with our theory.

Our study reveals an interesting phenomenon: depending on the algorithm hyperparameters, such as learning rate and batch-size, the resulting neural networks might possess different compressibility properties. When the step-size/batch-size ratio is large, heavy-tails emerge, and when combined with the decoupling effect of propagation of chaos that emerges with overparametrization, the networks become compressible in a way that they are amenable to simple pruning strategies. Finally, when compressible, the networks become likely to generalize better. In this sense, our results also provide an alternative perspective to the recent theoretical studies that suggest that heavy-tails can be beneficial in terms of generalization\footnote{In a concurrent study, \cite{shinCompressingHeavyTailedWeight2021} considered a related problem and examined the relationship between compressibility of heavy-tailed weight matrices and generalization from an alternative theoretical point of view, without an emphasis on the algorithmic aspects.} \cite{simsekli2020hausdorff,zhou2020towards}.

\section{Preliminaries and Technical Background}
\label{sec:prelims}
\textbf{Notation. } Matrices and vectors are denoted by upper- and lower-case bold letters, respectively,  \textit{e.g.} $\vc{X}$ and $\vc{x}$. A sequence of $n$ scalars $x_1,\ldots, x_{n}$ is shown as $\{x_i\}_{i=1}^n$. Similar notations are used for sequences of matrices $\left\{\vc{X}_i\right\}_{i=1}^{l}$ and vectors $\left\{\vc{x}_i\right\}_{i=1}^{l}$, whose entries are indexed with convention $\vc{X}_{i} {=} \left[\vc{X}_i\right]_{m,n}$ and $\vc{x}_{i} {=} \left(x_{i,1}, x_{i,2}, \ldots, x_{i,n}\right)$, respectively. The set of integers $\{m,\ldots,n\}$ is denoted by $\llbracket m,n \rrbracket$. We denote the $\ell_p$ (semi-)norm of a vector $\vc{x} \in \rset^d$ as $\|\vc{x}\|_p {=} (\sum_{i=1}^d |x_i|^p)^{1/p}$ for all $p \in (0,\infty)$ and $\|\vc{x}\|$ implies $ \|\vc{x}\|_2$.  For a matrix $\vc{A} \in \rset^{n\times m}$, $\|\vc{A}\|$ denotes its Frobenius norm.

\textbf{Fully connected neural networks. } In the entirety of the paper, we consider a multi-class classification setting. We denote the space of data points by $\mathcal{Z} =  \mathcal{X} \times \mathcal{Y}$, where $\mathcal{X} \subset \rset^{d_X}$ is the space of features and $\mathcal{Y} = \{1,2,\dots,d_Y\}$ is the space of the labels. Similar to prior art (e.g., \cite{neyshabur2017pac}), we focus our attention on the bounded feature domain $\mathcal{X} = \mathcal{X}_{B} := \left\{\vc{x} \in \mathbb{R}^{d_X}; \| \vc{x}\| \leq B \right\}$. 

We denote a \emph{fully connected neural network} with $L$ layers by a collection of \emph{weight matrices} $\{\vc{W}_{l}\}_{l=1}^L$, such that $\vc{W}_l \in \rset^{h_l \times h_{l-1}}$, where $h_l$ denotes the number of hidden units at layer $l$ with $h_0 = d_X$ and $h_L = d_Y$. Accordingly, the \emph{prediction function} $f_{\vc{w}}(\vc{x}): \mathcal{X}_{B} \mapsto \mathcal{Y}$ corresponding to the neural network is defined as follows:
\begin{equation}
f_{\vc{w}}(\vc{x})=\vc{W}_{L} \phi\left(\vc{W}_{L-1} \phi\left(\cdots \phi\left(\vc{W}_{1} {\vc{x}}\right) \right)\right),
\end{equation}
where $\phi : \rset \to \rset$ is the rectified linear unit (ReLU) activation function, \textit{i.e.}, $\phi(x) = \max(0,x)$, and it is applied element-wise when its input is a vector. For notational convenience, let us define $d_l\coloneqq h_l \times h_{l-1}$ and $d \coloneqq \sum_{l=1}^L d_l$. Furthermore, let  $\vc{w}_l$ denote the \emph{vectorized} weight matrix of layer $l$, \textit{i.e.},  $\vc{w}_l \coloneqq \textbf{vec}\left(\vc{W}_l\right) \in \rset^{d_l}$, where $\textbf{vec}$ denotes vectorization. Finally, let
$\vc{w}$ be the concatenation of all the vectorized weight matrices, i.e., $\vc{w} \coloneqq [\vc{w}_1, \dots, \vc{w}_L] \in \rset^d $. We assume that $L \geq 2$.

\textbf{Risk minimization and SGD.} 
In order to assess the quality of a neural network represented by its weights $\vc{w}$, we consider a loss function $\ell : \mathcal{Y} \times \mathcal{Y} \mapsto \rset_+$, such that $\ell({y}, f_{\vc{w}}(\vc{x}))$ measures the loss incurred by predicting the label of $\vc{x}$ as $f_{\vc{w}}(\vc{x})$, when the true label is $y$. By following a standard statistical learning theoretical setup, we consider an unknown \emph{data distribution} $\mu_Z$ over $\Zcal$, and a \emph{training dataset} with $n$ elements, \emph{i.e.}, $S = \{\vc{z}_1,\dots,\vc{z}_n\}$, where each $\vc{z}_i =: (\vc{x}_i, y_i) \stackrel{\mathclap{\mbox{\tiny{i.i.d.}}}}{\sim}  \mu_Z$.  We then denote the \emph{population} and \emph{empirical risks} as $\Rcal(\vc{w})\coloneqq\mathbb{E}_{(\vc{x},y)\sim\mu_Z}\left[\ell(y,f_{\vc{w}}(\vc{x}))\right]$ and $\widehat{\Rcal}(\vc{w})\coloneqq \frac{1}{n}\sum_{i=1}^n \ell(y_i,f_{\vc{w}}(\vc{x}_i))$.

Since $\mu_Z$ is unknown, we cannot directly attempt to minimize $\Rcal$ in practice. One popular approach to address this problem is the \emph{empirical risk minimization} strategy, where the goal is to solve the following optimization problem: $\min_{\vc{w\in\rset^d}} \widehat{\Rcal}(\vc{w})$. To tackle this problem, SGD has been one of the most popular optimization algorithms, which is based on the following simple recursion:
\begin{align}
\label{eqn:sgdrec}
\vc{w}^{\{k+1\}} = \vc{w}^{\{k\}} - \eta \nabla \widetilde{\Rcal}_{k+1} (\vc{w}^{\{k\}}),~\text{ where }~\nabla \widetilde{\Rcal}_k (\vc{w}) \coloneqq ({1}/{b}) \sum\nolimits_{i \in \Omega_k}  \nabla \ell\left(y_i,f_{\vc{w}}(\vc{x}_i)\right).
\end{align}
Here, $\vc{w}^{\{k\}}$ denotes the weight vector at iteration $k \in \mathbb{N}^+$, $\eta > 0$ is the step-size (or learning rate), $\nabla \widetilde{\Rcal}_k(\cdot)$ is the stochastic gradient, and $\Omega_k $ is the mini-batch with size $b$, drawn with or without replacement from $\llbracket 1,n\rrbracket$.

\textbf{Heavy-tailed distributions and the $\alpha$-stable family. } In this study, we will mainly deal with heavy-tailed random variables. While there exist different definitions of heavy-tails in the literature, here, we call a random variable heavy-tailed if its distribution function has a \emph{power-law} decay, \textit{i.e.}, $\mathbb{P}(X\geq x)\sim c x^{-\alpha}$ as $x \rightarrow \infty$, for some $c>0$ and $\alpha \in (0,2)$. Here, the \emph{tail-index} $\alpha$ determines the tail thickness of the distribution, \emph{i.e.}, as $\alpha$ get smaller, the distribution becomes heavier-tailed.

An important subclass of heavy-tailed distributions is the family of stable distributions. A random variable ${X}$ has symmetric $\alpha$-stable distribution, denoted by $\sas(\sigma)$, if its characteristic function is equal to $\mathbb{E}\left[\exp(iwX)\right]=\exp\left(-\left|\sigma w\right|^{\alpha}\right)$, where $\alpha \in (0,2]$ is again the tail-index and $\sigma \in (0,\infty)$ is the scale parameter.  An important property of $\sas$ is that  whenever $\alpha < 2$, $\mathbb{E}|X|^p$ is finite if and only if $p < \alpha$. This implies that the distribution has infinite variance as soon as $\alpha <2$. In addition to their wide use in applied fields \cite{Nolan2020}, recently, $\sas$ distributions have also been considered in deep learning theory \cite{pmlr-v97-simsekli19a,peluchetti20b,zhou2020towards,simsekli2020hausdorff}. 
\vspace{-3pt}
\section{Compressibility and the Heavy-Tailed Mean-Field Regime}
\label{sec:comp_and_meanfield}

In this section, we will present our first set of theoretical results. We first identify a sufficient condition, then we prove that, under this condition, the compression errors of different pruning techniques become arbitrarily small as the network size increases.

\vspace{-3pt}
\subsection{The heavy-tailed mean-field regime} 

Due to the peculiar generalization behavior of neural networks trained with SGD, recent years have witnessed an extensive investigation of the theoretical properties of SGD in deep learning \cite{london2017pac,dziugaite2017computing,kuzborskij2017data,zhou2019understanding,allen2019learning,simsekli2020hausdorff,neu2021information}. We will now mention two recently established theoretical properties of SGD, and our main contribution is the observation that these two seemingly unrelated `phenomena' are closely linked to compressibility.   

\textbf{\emph{The heavy-tail phenomenon.}} Several recent studies \cite{martin2019traditional,pmlr-v97-simsekli19a,zhou2020towards,zhang2020adaptive} have empirically illustrated that neural networks can exhibit heavy-tailed behaviors when optimized with SGD. Theoretically investigating the origins of this heavy-tailed behavior, \cite{hodgkinson2020multiplicative} and \cite{gurbuzbalaban2020heavy} later proved several theoretical results on online SGD, with rather surprising implications: due to the `multiplicative' structure of the gradient noise (\emph{i.e.}, $\nabla \widetilde{\Rcal}_k(\vc{w}) - \nabla \widehat{\Rcal}(\vc{w})$), as $k\to \infty$, the distribution of the SGD iterates $\vc{w}^{\{k\}}$ can converge to a heavy-tailed distribution with \emph{infinite second-order moment}, and perhaps more surprisingly, this behavior can even emerge in simple linear regression with Gaussian data. The authors of \cite{gurbuzbalaban2020heavy} further showed that, in the linear regression setting, the tail index $\alpha$ is \emph{monotonic} with respect to the step-size $\eta$ and batch-size $b$: larger $\eta$ and/or smaller $b$ result in smaller $\alpha$ (\emph{i.e.}, heavier tails). 

\textbf{\emph{Propagation of chaos and mean-field limits.}} Another interesting property of SGD appears when the size of the network goes to infinity. Recently, it has been shown that, under an appropriate scaling of step-sizes, the empirical distribution of the network weights converges to a fixed \emph{mean-field} distribution, and the SGD dynamics can be represented as a \emph{gradient flow} in the space of probability measures \cite{mei2018mean,chizat2018global,sirignano2020mean,de2020quantitative}. Moreover, \cite{de2020quantitative} showed that a \emph{propagation of chaos} phenomenon occurs in this setting, which indicates that when the network weights are initialized independently (\emph{i.e.,} a-priori chaotic behavior), \emph{the weights stay independent} as the algorithm evolves over time (\emph{i.e.}, the chaos propagates) \cite{sznitman1991topics}. 

Our interest in this paper is the case where both of these phenomena occur \emph{simultaneously}: as the size of the network goes to infinity, the SGD iterates keep being \emph{independent} and their distribution converges to a \emph{heavy-tailed} distribution. To formalize this setting, let us introduce the required notations. For an integer $m_l < d_l$, let $\vc{w}_{l,(1:m_l)}$ denote the first $m_l$ coordinates of $\vc{w}_{l}$, and for $m_l\geq d_l$ set $\vc{w}_{l,(1:m_l)} = \vc{w}_{l}$. Furthermore, parametrize the dimension of each layer $d_l$ with a parameter $\rho \in \rset$, i.e., $d_l = d_l(\rho)$, such that for all $l =1,\dots, L$, we have:
\begin{align}
\label{eqn:dim_rho}
    d_l(\rho) \in \mathbb{N}_+, \quad \text{and} \quad  \lim\nolimits_{\rho \to \infty} d_l(\rho) = \infty.
\end{align}
This construction enables us to take the dimensions of each layer to infinity simultaneously \cite{nguyen2020rigorous}. 

\begin{definition}[The heavy-tailed mean-field limit property]
\label{cond:htmf}
The SGD recursion \eqref{eqn:sgdrec} possesses the heavy-tailed mean-field limit (HML) property, if the dimensions obey \eqref{eqn:dim_rho} and if there exist heavy-tailed probability measures $\{\mu_l^\star\}_{l=1}^L$ on $\rset$ with tail-indices $\{\alpha_l\}_{l=1}^L$, such that $\mu_l^\star(\{0\}) = 0$ for all $l$ and for any $m_1,\dots,m_L \in \mathbb{N}_+$, the joint distribution of
\begin{align}
\left(\vc{w}^{\{k\}}_{1,(1:m_1)},\dots, \vc{w}^{\{k\}}_{L,(1:m_L)} \right) \quad \text{weakly converges to} \quad  (\mu_1^\star)^{\otimes m_1} \otimes \cdots \otimes (\mu_L^\star)^{\otimes m_L}, \label{eqn:mhl}
\end{align}
as $\rho,k \to \infty$, where $\mu \otimes \nu$ denotes the product measure and $\mu^{\otimes n}$ denotes the $n$-fold product measure. 
\end{definition}
Informally, the HML property states that in the infinite size and infinite iteration limit, the entries of the weight matrices will become independent, and the distribution of the elements within the same layer will be identical and heavy-tailed. We will briefly illustrate this behavior empirically in Section~\ref{sec:exps}, and present additional evidence in the appendix. Note that, in \eqref{eqn:mhl}, the particular form of independence in the limit is not crucial and we will discuss weaker alternatives in Section~\ref{sec:hml}. 

We further note that \cite{de2020quantitative} proved that a similar form of \eqref{eqn:mhl} indeed holds for SGD applied on single hidden-layered neural networks, where the limiting distributions possess second-order moments (\emph{i.e.}, not heavy-tailed) and the independence is column-wise. Recently, \cite{peluchetti20b} investigated the infinite width limits of fully connected networks initialized from a $\sas$ distribution and proved heavy-tailed limiting distributions. On the other hand, heavy-tailed propagation of chaos results have been proven in theoretical probability \cite{jourdain2008nonlinear,liang2019exponential}; however, their connection to SGD has not been yet established. We believe that \eqref{eqn:mhl} can be shown to hold under appropriate conditions, which we leave as future work. 

\subsection{Analysis of compression algorithms} 

In this section, we will analyze the compression errors of three different compression schemes when the HML property is attained. All three methods are based on \emph{pruning}, which we formally define as follows. Let $\vc{x}$ be a vector of length $d$, and consider its sorted version in descending order with respect to the magnitude of its entries, \emph{i.e.}, $|x_{i_1}|\geq |x_{i_2}|\geq \cdots \geq |x_{i_d}|$, where $\{i_1,\ldots,i_d\}=\{1,\ldots,d\}$. For any $k \leq d$, the $k$-best term approximation of $\vc{x}$, denoted as $\vc{x}^{(k)}=(x^{(k)}_1,\ldots,x^{(k)}_d)$, is defined as follows: for $l \in \llbracket 1,\lceil k\rceil\rrbracket$, $x^{(k)}_{i_l} \coloneqq x_{i_l}$   and for $l \notin \llbracket 1,\lceil k\rceil\rrbracket$, $x^{(k)}_{i_l} \coloneqq 0$. Informally, we keep the $k$-largest entries of $\vc{x}$ with the largest magnitudes, and `prune' the remaining ones. Current results pertain to one-shot pruning with no fine-tuning; other settings are left for future work (cf.\ \cite{elesedy2020lottery}).

In this section, we consider that we have access to a sample from the stationary distribution of the SGD, \emph{i.e.} $\vc{w}^{\{\infty\}}$, and for conciseness we will simply denote it by $\vc{w}$. We then consider a \emph{compressed network} $\widehat{\vc{w}}$ (that can be obtained by different compression schemes) and measure the performance of the compression scheme by its `relative $\ell_p$-compression error' (cf.\ \cite{amini2011compressibility,gribonval2012compressible}), defined as: $\|\widehat{\vc{w}}-\vc{w}\|_p/\|\vc{w}\|_p$. Importantly for the following results, in the appendix, we further prove that a small compression error also implies a small perturbation on the network output.

\textbf{\emph{Magnitude pruning.}} Magnitude(-based) pruning has been one of the most common and efficient algorithms among all the network pruning strategies \cite{Han2015,blalock2020state,kuhn2021robustness}. In this section, we consider the global and layer-wise magnitude pruning strategies under the HML property.  

More precisely, given a network weight vector $\vc{w} \in \rset^d$ and a \emph{remaining parameter ratio} $\kappa \in (0,1)$, the global pruning strategy compresses $\vc{w}$ by using $\vc{w}^{(\kappa d)}$, \emph{i.e.}, it prunes the smallest (in magnitude) $(1-\kappa)d$ entries of $\vc{w}$. Also,  note that $1/\kappa$ corresponds to the frequently used metric \textit{compression rate} \cite{blalock2020state}. On the other hand, the layer-wise pruning strategy applies the same approach to each layer separately, \emph{i.e.}, given layer-wise remaining parameter ratios $\kappa_l\in (0,1)$, we compress each layer weight $\vc{w}_l \in \rset^{d_l}$ by using $\vc{w}_l^{(\kappa_l d_l)}$. The following result shows that the compression error of magnitude pruning can be made arbitrarily small as the network size grows.  
\begin{theorem} \label{th:magPruning} Assume that the recursion \eqref{eqn:sgdrec} has the HML property. 
\begin{enumerate}[label=(\roman*),noitemsep,topsep=0pt,leftmargin=*,align=left]
    \item \underline{Global magnitude pruning:} if the weights of all layers  have identical asymptotic distributions $ \mu_l^\star \equiv \mu^\star$ with tail-index $\alpha_l^\star = \alpha  $, for all $l \in \llbracket 1,L \rrbracket$, then for every $\epsilon>0$, $\veps >0$, $\kappa \in (0,1)$, and $p \geq \alpha$, there exists $d_0 \in \mathbb{N}$, such that $\|\vc{w}^{(\kappa d)} - \vc{w}\|_p\leq \veps \|\vc{w}\|_p $ holds with probability at least $1-\epsilon$,  for $d \geq d_0$.  \vspace{3pt}
    \item \underline{Layer-wise magnitude pruning:} for every $\epsilon>0$, $\veps_l >0$ and $\kappa_l \in (0,1)$, where $l \in \llbracket 1,L\rrbracket$, and $p \geq \max_l\alpha_l$, there exists $d_{l,0} \in \mathbb{N}$, such that  $\|\vc{w}_l^{(\kappa_l d_l)} - \vc{w}_l\|_p\leq \veps_l \|\vc{w}_l\|_p$ holds with probability at least $1-\epsilon$, for $d_{l} \geq d_{l,0}$.
\end{enumerate}
\vspace{-6pt}
\end{theorem}
This result shows that any pruning ratio of the weights is achievable as long as the network size is sufficiently large and the network weights are close enough to an i.i.d.\ heavy-tailed distribution. Empirical studies report that global magnitude pruning often works better than layer-wise magnitude pruning \cite{blalock2020state}, except when it leads to over-aggressive pruning of particular layers \cite{wangPickingWinning2020}.

The success of this strategy under the HML property is due to a result from compressed sensing theory, concurrently proven in \cite{gribonval2012compressible,amini2011compressibility}, which informally states that for a large vector of i.i.d.\ heavy-tailed random variables, the norm of the vector is mainly determined by a small fraction of its entries. We also illustrate this visually in the appendix. 

An important question here is that, to achieve a fixed relative error $\veps_l$, how would the smallest $\kappa_l$ would differ with varying tail-indices $\alpha_l$, \emph{i.e.}, whether ``heavier-tails imply more compression''. In our experiments, we illustrate this behavior positively: heavier-tailed weights are indeed more compressible. We partly justify this behavior for a certain range of $p$ in the appendix; however, a more comprehensive theory is needed, which we leave as future work.     

\textbf{\emph{Singular value pruning.} } In recent studies, it has been illustrated that the magnitudes of the eigenvalues of the sample covariance matrices (for different layers) can decay quickly, hence pruning the singular values of the weight matrices, \emph{i.e.}, only keeping the largest singular values and corresponding singular vectors, is a sensible pruning method. Exploiting the low-rank nature of fully connected and convolutional layer weights in network compression has been investigated theoretically and empirically \cite{arora2018stronger, yuCompressingDeepModels2017}. Here we will present a simple scheme to demonstrate our results. 

More precisely, for the weight matrix at layer $l$, $\vc{W}_l$, consider its singular value decomposition, $\vc{W}_l = \vc{U} \boldsymbol{\Sigma} \vc{V}^\top$, and then, with a slight abuse of notation, define $\vc{W}^{[\kappa_l h_{l-1}]}_l := \vc{U} \boldsymbol{\Sigma}^{(\kappa_l h_{l-1})} \vc{V}^\top$, where $\boldsymbol{\Sigma}^{(\kappa_l h_{l-1})}$ is the diagonal matrix whose diagonal entries contain the $\lceil \kappa_l h_{l-1} \rceil$ largest singular values (\emph{i.e.}, prune the diagonal of $\boldsymbol{\Sigma}$). Accordingly, denote $\vc{w}^{[\kappa_l h_{l-1}]}_l := \textbf{vec}(\vc{W}^{[\kappa_l h_{l-1}]}_l)$. 

The next theorem shows that under the HML property with an additional requirement that the limiting distributions are $\sas$, the eigenvalues of the (properly normalized) sample covariance matrices will be indeed compressible and the pruning strategy achieves negligible errors as the network size grows.
\begin{theorem}
\label{th:spectral}
Assume that the recursion \eqref{eqn:sgdrec} has the HML property, $L\geq 3$, and  for all $l \in \llbracket 2, L-1 \rrbracket$, $\mu^\star_l \equiv  \mathcal{S}\alpha_l\mathcal{S}(\sigma_l) $ with some $\sigma_l >0$. Then, for every $\epsilon>0$, $\veps_l > 0$, and $\kappa_l \in (0,1)$, there exists $\{h_{l,0,}\}_{l=1}^{L}$, such that the following inequalities hold for every $ h_{l} \geq h_{l,0}$ and $p\geq \max_l \alpha_l/2$:
\begin{align}
{\|\boldsymbol{\lambda}_l^{(\kappa_l h_{l-1})} - \boldsymbol{\lambda}_l  \|_p } \leq \veps_l^2{\| \boldsymbol{\lambda}_l \|_p} , \quad \text{and} \quad \|\vc{w}^{[\kappa_l {h_{l-1}}]}_l - \vc{w}_l\|\leq \veps_l \| \vc{w}_l\|,
\end{align}
 with probability at least $1-\epsilon$, where $ \boldsymbol{\lambda}_l \in \rset^{h_{l-1}}$ denotes the vector of eigenvalues corresponding to the sample covariance matrix $h_{l}^{-2/\alpha_l}\vc{W}_l^\top \vc{W}_l$.
\end{theorem}
We shall note that the proof of Theorem~\ref{th:spectral} in fact only requires the limiting distributions to be `regularly varying' \cite{teimouri2020asymptotic} and symmetric around zero, which covers a broad range of heavy-tailed distributions beyond the $\alpha$-stable family \cite{buraczewski2016stochastic}. The sole reason why we require the $\sas$ condition here is to avoid introducing further technical notation.

\textbf{\emph{Node pruning.} } The last pruning strategy that we consider is a structured pruning strategy, that corresponds to the removal of the whole  columns of a fully connected layer weight matrix. Even though below we consider pruning based on the norms of the weight layer columns, the same arguments apply for pruning rows. 
 
The idea in column pruning is that, for a given layer $l$, we first sort the columns of the weight matrix $\vc{W}_l \in \rset^{h_l \times h_{l-1}}$ with respect to their $\ell_p$-norms for a given $p \geq \max_l \alpha_l$. Then, we remove the \emph{entire columns} that have the smallest $\ell_p$-norms. More precisely, let $\vc{W}_l(i) \in \mathbb{R}^{h_l}$ be the $i$-th column of $\vc{W}_l$, for $i \in \llbracket 1, h_{l-1} \rrbracket$, and suppose that $\|\vc{W}_l(i_1)\|_p \geq \|\vc{W}_l(i_2)\|_p\geq \cdots \geq \|\vc{W}_l(i_{h_{l-1}})\|_p$, where $\{i_1,\ldots,i_{h_{l-1}}\}=\{1,\ldots,h_{l-1}\}$. Then, we define the $k$-best column approximation of $\vc{W}_l$, denoted as $\vc{W}_l^{\{ k \},p} \in \mathbb{R}^{h_{l} \times h_{l-1}}$, as follows: for $j \in \llbracket 1,\lceil k\rceil\rrbracket$, $\vc{W}_l^{\{ k \},p}(i_j) \coloneqq \vc{W}_l(i_j)$ and for $j \notin \llbracket 1,\lceil k\rceil\rrbracket$, $\vc{W}_l^{\{ k \},p}(i_j) \coloneqq 0$. Denote also $\vc{w}_l^{\{ k \},p}\coloneqq \textbf{vec}\left(\vc{W}_l^{\{ k \},p}\right)$. 

\begin{theorem} \label{th:normPruning} Assume that the recursion \eqref{eqn:sgdrec} has the HML property. Then, for every $\epsilon>0$, $\veps_l >0$, $\kappa_l \in (0,1)$, where $l \in \llbracket 2,L \rrbracket$, and $p\geq \max_l \alpha_l$,  there exists $h_{l-1,0} \in \mathbb{N}$, such that  $\|\vc{w}_l^{\{\kappa_l h_{l-1}\},p} - \vc{w}_l\|_p\leq \veps_l \|\vc{w}_l\|_p $ holds with probability at least $1-\epsilon$, for every $h_{l-1} \geq h_{l-1,0}$.
\end{theorem}
This theorem indicates that we can remove entire columns in each layer, without considerably affecting the network weights, as long as the network is large enough. In other words, effectively the widths of the layers can be reduced. Structured pruning schemes are commonly used in CNNs where filters, channels, or kernels can be pruned by norm-based or other criteria \cite{liPruningFiltersEfficient2017, heChannelPruningAccelerating2017, Han2015}.

\subsection{A note on the limiting independence structure in the HML property }
\label{sec:hml}

We conclude this section by discussing the particular independence condition in the limit, which appears in Definition~\ref{cond:htmf}. We shall underline that the element-wise independence is not a necessity and under weaker conditions, we can still obtain Theorems~\ref{th:magPruning}-\ref{th:normPruning} with identical or almost identical proofs. For instance, the proof of Theorem~\ref{th:normPruning} remains the same when the columns of $\vc{W}_l$ are i.i.d.\ vectors with \emph{dependent} components; hence, the element-wise independence is indeed not needed, but is used for the clarity of presentation. More generally, in all three theorems, the main requirement is to ensure a \emph{weak dependence} between the components of the weight vector. More precisely, for Theorems~\ref{th:magPruning}, \ref{th:spectral}, \ref{th:normPruning}, we respectively need (1) the entries of $\vc{W}_l$ or (2) its singular values, or (3) the $\ell_p$-norms of its columns to be \emph{stationary} and \emph{ergodic} with a heavy-tailed distribution. Under this condition, the same proof strategies will still work by invoking \cite[Theorem 1]{silva2015characterization}, instead of \cite[Proposition 1]{gribonval2012compressible}.

\section{Generalization Bounds} \label{sec:generalization}

So far, we have shown that the heavy-tailed behavior in the weights of a network together with their independence result in compressibility. In this section, we further show that these phenomena bring forth a better generalization performance in the network. More precisely, we show that if a network is more compressible, then the corresponding compressed network has a  smaller generalization error. 
Throughout the section, we will focus on layer-wise magnitude pruning; yet, our results can be easily extended to the other pruning strategies. Note that similar results have already been proven in \cite{arora2018stronger,suzuki2018spectral,Suzuki2020Compression,hsu2021generalization}; yet, they cannot be directly used in our specific setting, hence, we prove results that are customized to our setup.

Our generalization bounds are derived by applying the previously developed techniques as in \cite{arora2018stronger,neyshabur2017pac}. In particular, we follow the approach of  \cite{arora2018stronger} by further adapting the technique for the magnitude pruning strategy and allowing the compressed network weights to take unbounded continuous values. As in \cite{arora2018stronger}, we consider the $0{\text{-}}1$ loss function with \emph{margin} $\gamma \geq 0$,  $\ell_{\gamma}\colon \mathcal{Y}\times \mathcal{Y} \mapsto \{0,1\}$, for the multiclass classifier $f_{\vc{w}}$, given as follows:
\begin{align}
    \ell_{\gamma}(y,f_{\vc{w}}(\vc{x}))=\begin{cases}1,& \text{ if } f_{\vc{w}}(\vc{x})[y]-\max \limits_{j \neq y} f_{\vc{w}}(\vc{x})[j] \leq \gamma,\\
    0,& \text{ otherwise}.
    \end{cases}
\end{align}
Still denoting $\vc{w}=\vc{w}^{\{\infty \}}$, the population and empirical risks associated with $\ell_\gamma$ are denoted as $\Rcal_\gamma(\vc{w})$ and $\widehat{\Rcal}_\gamma(\vc{w})$, respectively. By having the dataset $S\sim \mu_Z^{\otimes n}$, we  assume the weights $\vc{w}$ are sampled according to the stationary distribution $P_{\vc{w}|S}$ of SGD. Denote the joint distribution of $(S,\vc{w})$ by  $P_{S,\vc{w}}\coloneqq \mu_Z^{\otimes n} P_{\vc{w}|S}$.

In the following theoretical results, we will assume that we have access to a random compressible neural network that is  amenable to the layer-wise magnitude pruning strategy with high probability. 
This assumption is essentially the outcome of Theorem~\ref{th:magPruning} under the HML property together with an additional uniformity requirement on $d_{l,0}$ over $S$\footnote{Note that the uniformity assumption is mild and can be avoided by combining Theorem~\ref{th:magPruning} with Egoroff's theorem with additional effort, as in \cite{simsekli2020hausdorff}.}. 
\begin{assumption}
\label{asmp:magp}
For $\veps \geq 0$, $\epsilon>0$, and  $\{\kappa_l\}_{l=1}^L\colon \kappa_l \in (0,1)$,  there exists  $\{d_{l,0}\}_{l=1}^L \colon d_{l,0} \in \mathbb{N}$ independent of $S$,   such that for $d_l \geq d_{l,0}$, $l \in \llbracket 1,L \rrbracket$, the relation $\|\vc{\widehat{w}}_l - \vc{w}_l\| \leq \veps \|\vc{w}_l\|$ holds  for all $l \in \llbracket 1,L \rrbracket$ simultaneously with probability at least $1-\epsilon$, where  $\vc{\widehat{w}}_l:=\vc{w}_l^{\left(\kappa_l d_l\right)}$ and the probability is w.r.t.\ $P_{S,\vc{w}}$.
\end{assumption}

In  the following result, we relate the population risk of the compressed network to the empirical margin risk of the original (uncompressed) network. For notational convenience, for $\delta,\tau>0$, let
\begin{align}
        R(\delta) \coloneqq \inf \left\{ R \colon \mathbb{P}\left( \|\vc{w}\| \geq R \right) \leq \delta \right\},~\text{ and }~  \mathcal{L}(\tau,\delta) \coloneqq \sqrt{2}BL\left(2 R(\delta)/\sqrt{L}\right)^{L-1}/\tau, \label{def:RofDeltaLipchtz}
\end{align}
where the probability is with respect to the joint distribution $P_{S,\vc{w}}$.

\begin{theorem} \label{th:generalizationPrunedNetwork}
Assume \Cref{asmp:magp} holds. Then for $n\colon n/\log(n) \geq 10L$, $\{d_l\}_{l=1}^L\colon d_l\geq d_{l,0}$, and any $ \delta,\tau>0 $, with probability  at least $1-2e^{-\kappa d/2}-\delta-\epsilon$,
\begin{align}
         \Rcal_{0}({\vc{\widehat{w}}}) & \leq \widehat{\Rcal}_{\gamma(\delta,\tau)}({\vc{w}})+ \left(12 \mathcal{L}(\tau,\delta)R(\delta)+\sqrt{d}\right)\sqrt{\left(\kappa+\epsilon_{\kappa}\right) \log(n)/n}, \label{eq:generalizationPrunedNetwork}
\end{align}
where $R(\delta)$ and $\mathcal{L}(\tau,\delta)$ are defined in   \eqref{def:RofDeltaLipchtz},  $\kappa\coloneqq \frac{1}{d}\sum_{l=1}^L \lceil \kappa_l d_l\rceil$,\footnote{The binary entropy $h_b(\kappa)$ (in nats) is defined as $-\kappa \log(\kappa)-(1-\kappa) \log(1-\kappa)$ for $\kappa \in [0,1]$, with the convention $h_b(0)=h_b(1)=0$. Note that $0 \leq h_b(\kappa)\leq \log(2)$.}
\begin{align*}
   \epsilon_{\kappa} \coloneqq \left(2 h_b(\kappa)-\kappa\log \left(\kappa\right)\right)/\log(n),~\text{ and }\gamma(\delta,\tau) \coloneqq  \tau+\frac{\sqrt{2} B  }{\tau}\left(R(\delta)/\sqrt{L}\right)^L \left((1+\varepsilon)^{L}-1\right). 
\end{align*}
\end{theorem}

This result shows that, for a fixed relative compression error $\varepsilon$, if a network is more compressible, \textit{i.e.} if $\kappa$ is smaller, then its corresponding compressed network has a smaller generalization error bound, as the bound is scaled by the factor ${\approx} \sqrt{\kappa}$. 

In our proof, we prove an intermediate result, inspired by \cite{neyshabur2015norm} and stated in the appendix, which bounds the perturbation of the network output after compression. This guarantees that the risks of the original and pruned networks can be made arbitrarily close, as long as the relative compression error is small enough. 

To have a better understanding of the constants in Theorem~\ref{th:generalizationPrunedNetwork}, in the next result, we consider a special case where the weights follow a stable distribution, and we make the above bound more explicit.

\begin{cor} \label{cor:genCompressStable} Assume that for $l \in \llbracket 1,L \rrbracket$ and $i \in \llbracket 1,d_l \rrbracket$, the conditional distribution of ${w}_{l,i} \stackrel{\mathclap{\mbox{\tiny{i.i.d.}}}}{\sim} \mathcal{S}\alpha_l \mathcal{S}(\sigma_l) $  with $\alpha_l \in (1,2)$. Further, assume that the scale parameters satisfy the following property:
\begin{align}
\sigma^2 \coloneqq \sum\nolimits_{l=1}^L (d_l/d) (\sigma_l/\sigma_{\alpha_l})^2= \left[(4^{-1/\alpha}\sqrt{L}/3) \sigma_0 d^{-(1/2+1/\alpha)}\right]^2 ,
\end{align}
where $\sigma_{\alpha_l}\coloneqq \left(2\Gamma(-\alpha_l)\cos((2-\alpha_l)\pi/2)\right)^{1/\alpha_l}$, $\alpha {\coloneqq}\min_{l}\alpha_l$, and  $\sigma_0$ is a constant, and also  $\{\alpha_l\}_{l=1}^L$ and $\sigma$ are independent from $S$. Then, for every $\veps >0$ and $\kappa_l \in (0,1)$, $l \in \llbracket 1, L \rrbracket$,  there exists $d_{l,0} \in \mathbb{N}$,  such that  for $d_l \geq d_{l,0}$,   $n \colon n/\log(n) \geq 10L$, and every $\tau>0$, with probability  at least $1-3d^{-\alpha/(2L)}$,
\begin{align}
         \Rcal_{0}({\vc{\widehat{w}}}) & \leq \hat{\Rcal}_{\gamma}({\vc{w}})+ \left(a \sigma_0^L/\tau+1\right)\sqrt{\left(\kappa+\epsilon_{\kappa}\right) d \log(n)/n}, \label{eq:generalizationPrunedNetworkStable}
\end{align}
where $\{\vc{\widehat{w}_l}\}_l=\{\vc{w}_l^{(\kappa_l  d_l)}\}_l$, $\gamma \coloneqq \tau+ b_{\varepsilon} \sigma_0^{L} \sqrt{d}/\tau$, $a \coloneqq  6 \sqrt{2} B  2^{L} L^{3/2} $, and  $b_{\varepsilon} \coloneqq \sqrt{2} B   \left((1+\varepsilon)^{L}-1\right)$. 
\end{cor}
To simplify our presentation, we have set the `scale' of the distribution as a decreasing function of dimension, which intuitively states that the typical magnitude of each entry of the network will get smaller as the network grows. We observe that for a fixed $\varepsilon$ and $d$, the bound improves as pruning ratio, $1-\kappa$, increases. This result is of interest in particular since it is observed experimentally (and in part, theoretically) that heavier-tailed weights are more compressible, and hence due to this result have better generalization bounds. This provides an alternative perspective to the recent bounds that aim at linking heavy tails to generalization through a geometric approach \cite{simsekli2020hausdorff}.

Finally, in the appendix, we further show that the uncompressed network also inherits this good generalization performance, which is consistent with the results of \cite{hsu2021generalization,kuhn2021robustness}. 

\section{Experiments}
\label{sec:exps}

\begin{figure}[t]
    \centering
    \includegraphics[width=\textwidth]{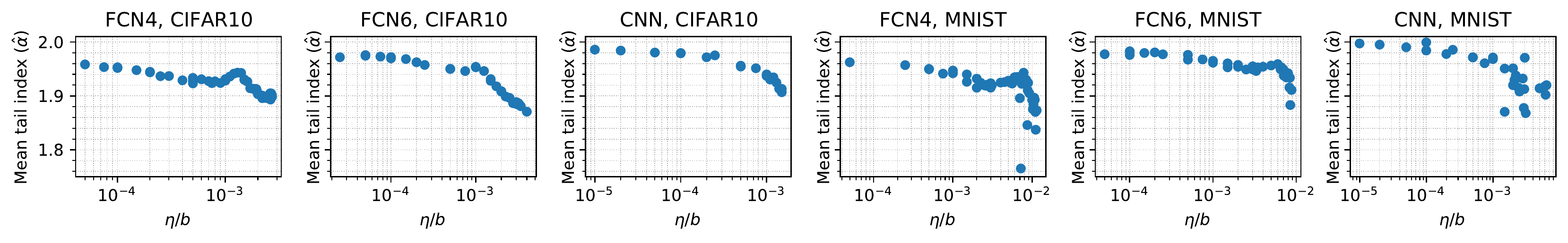}
    \caption{Mean estimated tail index ($\hat{\alpha}$) vs. $\eta/b$ for each trained model. $x$-axes are log-scaled.}
    \label{fig:alpha_vs_eta_b} 
\end{figure}

In this section, we present experiments conducted with neural networks to investigate our theory. We use three different model architectures: 
a fully connected network with $4$ hidden layers (FCN4), a fully connected network with $6$ hidden layers (FCN6), and a convolutional neural network with $8$ convolutional layers (CNN). Hidden layer widths were $2048$ for both FCN models. All networks include ReLU activation functions and none include batch normalization, dropout, residual layers, or any explicit regularization term in the loss function. Each model is trained on MNIST \cite{mnist2010} and CIFAR10 \cite{cifar102009} datasets under various hyperparameter settings, using the default splits for training and evaluation. The total number of parameters were approximately $14M$ for FCN4-MNIST, $19M$ for FCN4-CIFAR10, $23M$ for FCN6-MNIST, $27M$ for FCN6-CIFAR10, and $9M$ for both CNN models. All models were trained with SGD until convergence with constant learning rates and no momentum. The convergence criteria comprised $100\%$ training accuracy and a training negative log-likelihood less than $5 \times 10^{-5}$. The training hyperparameter settings include two batch-sizes ($b = 50, 100$) and various learning rates ($\eta$) to generate a large range of $\eta/b$ values. See the appendix for more details. 
\begin{wrapfigure}{r}{0.5\textwidth}
    \centering
    \includegraphics[width=.49\textwidth]{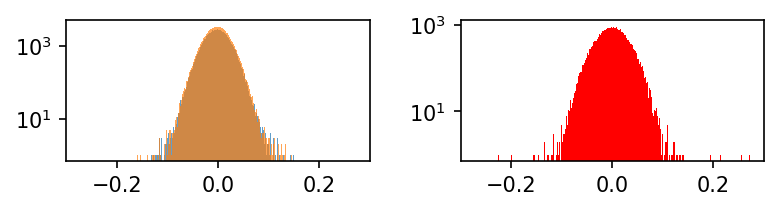}
    \captionof{figure}{ \small{Empirical distribution of a CNN layer trained on MNIST. (Left) Overlaid histograms of a random partition of the weights, showing an identical distribution. (Right) Comparing the network weights to samples simulated i.i.d. from a symmetric $\alpha$-stable distribution with the same tail index $\alpha  \approx 1.95$. $y$-axes are log-scaled. } } \label{fig:layer_histogram}
\end{wrapfigure}
By invoking \cite[Corollary 11]{gurbuzbalaban2020heavy} which shows the ergodic average of heavy-tailed SGD iterates converges to a multivariate stable distribution, after convergence, we kept running SGD  for $1000$ additional iterations to obtain the average of the parameters to be used in this estimation. The tail index estimations were made by the estimator proposed in \cite[Corollary 
2.4]{bgindex}, which has been used by other recent studies that estimate tail index in neural network parameters \cite{simsekli2020hausdorff,gurbuzbalaban2020heavy,zhou2020towards}. In all pruning methods, the parameters were centered before pruning is conducted with the median value to conform with the tail index estimation.

\textbf{Training hyperparameters and layer statistics.} We first verify that models trained with higher learning rate to batch-size ratio ($\eta / b$) lead to heavier-tailed parameters, replicating the results presented in \cite{gurbuzbalaban2020heavy}. 
For each trained model, we compute the mean of separately estimated tail indices for all layers, so that each model is represented by a single value. This averaging is a heuristic without a clear theoretical meaning, and has been also used by recent works in the literature \cite{gurbuzbalaban2020heavy,martin2019traditional}. Results presented in Figure \ref{fig:alpha_vs_eta_b} demonstrates confirm that higher $\eta/b$ leads to heavier tails (lower $\hat{\alpha}$).

Also of interest are the distribution of the resulting parameters from training. Figure \ref{fig:layer_histogram} (left) demonstrates a representative example of the parameters from an CNN layer trained on MNIST, where two overlaid histograms representing empirical distributions of a random partition of the parameters are almost identical. On the right in the same figure, this distribution is compared to a simulated symmetric $\alpha$-stable distribution that has the same tail index ($\approx 1.95$). The figure demonstrates that the two distributions have similar qualitative properties.

\begin{figure}[h]
    \centering
    \includegraphics[width=\textwidth]{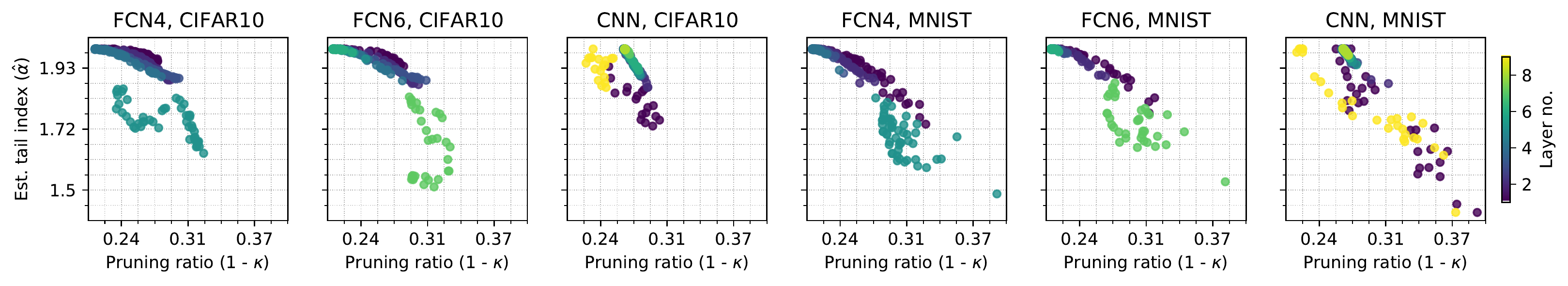}
    \caption{Estimated tail index ($\hat{\alpha}$) vs. pruning ratio ($1 - \kappa$), for relative compression error $= 0.1$}
    \label{fig:alpha_kappa}
\end{figure}

\textbf{Tail index and prunability.} In this section we examine whether networks with heavier-tailed layers, trained with higher $\eta/b$ ratios, are more prunable and whether they generalize better. As a baseline test, we first examine whether neural network layers which are heavier-tailed can be pruned more given a fixed maximum relative compression error. Figure \ref{fig:alpha_kappa} demonstrates for an error of $0.1$ that this is indeed the case. We next test our hypothesis that posits models with heavy-tailed parameters to be more prunable. Both the results pertaining to layer-wise magnitude pruning and singular value pruning, demonstrated in Figures \ref{fig:magnitude_pruning} and \ref{fig:spectral_pruning} show that this is indeed the case. Here, relative test accuracy stands for test accuracy after pruning / unpruned test accuracy.  The results show that models with heavier-tailed parameters (shown with darker colors) are starkly more robust against pruning. Similar results with global magnitude pruning can be seen in the appendix.
\begin{figure}[t]
    \centering
    \includegraphics[width=\textwidth]{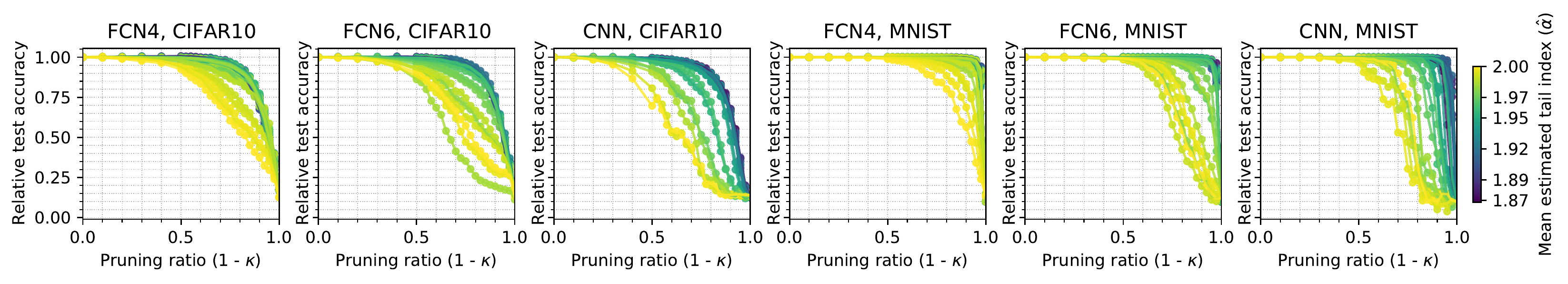}
    \caption{Relative test accuracy vs. pruning ratio for layer-wise magnitude pruning. Color: mean $\hat{\alpha}$.}
    \label{fig:magnitude_pruning}
\end{figure}
\begin{figure}[t]
    \centering
    \includegraphics[width=\textwidth]{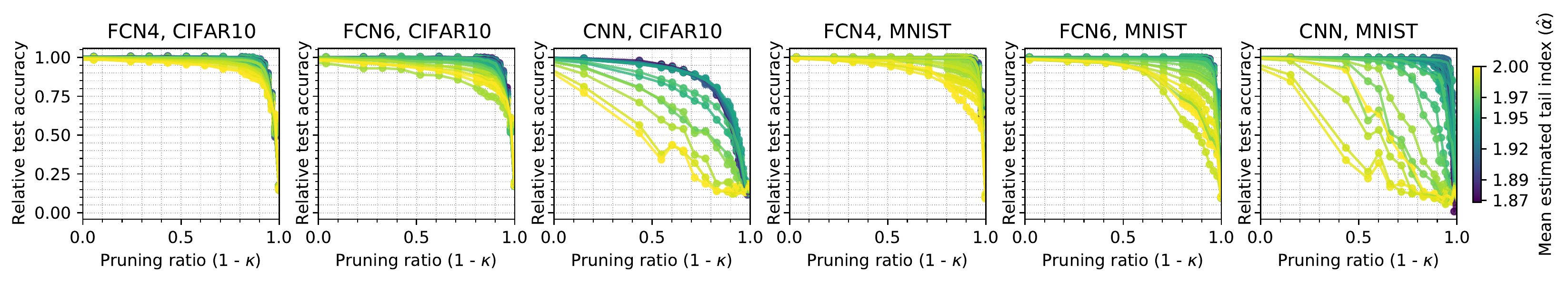}
    \caption{Relative test accuracy vs. pruning ratio for singular value pruning. Color: mean $\hat{\alpha}$.}
    \label{fig:spectral_pruning}
\end{figure}
\begin{wrapfigure}{r}{0.45\textwidth}
    \includegraphics[width=0.42\textwidth]{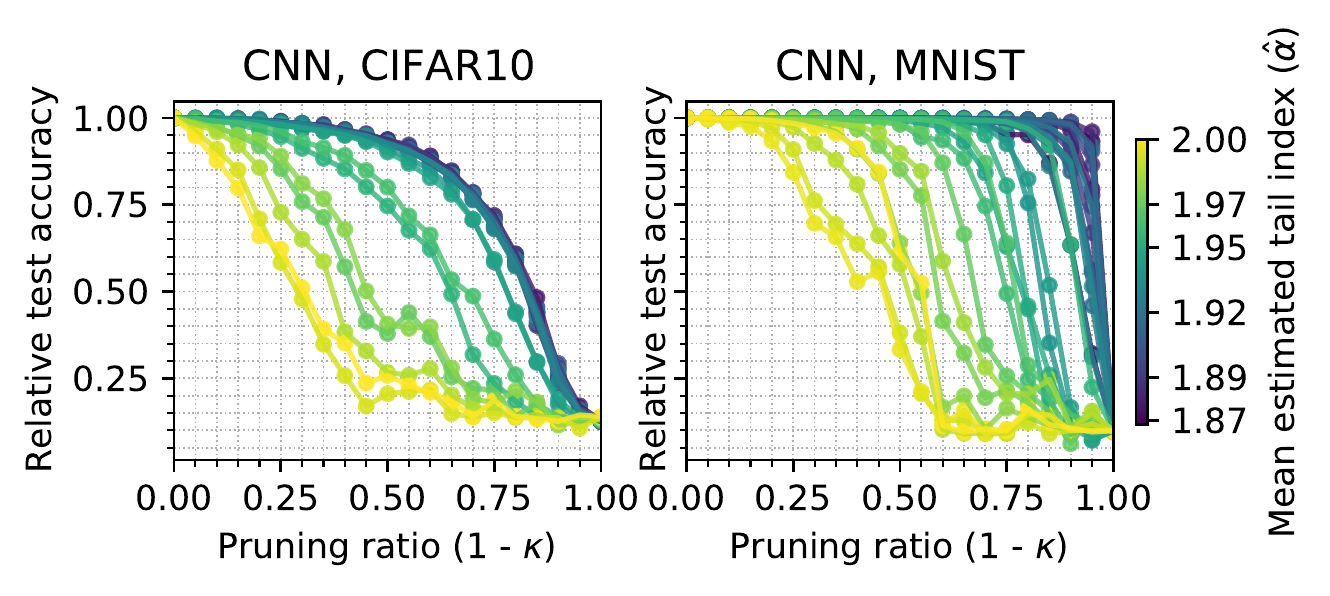}
    \captionof{figure}{Relative test accuracy vs. pruning ratio for structured pruning. Color: mean $\hat{\alpha}$.} \label{fig:nb_pruning}
\vspace{-5pt}
\end{wrapfigure}

For structured pruning, we prune $3\times 3$ kernel parameters in CNN models. The results (Figure \ref{fig:nb_pruning}) show a similar, hypothesis conforming pattern. Results for node pruning in FCNs, presented in the appendix, were underwhelming; perhaps unsurprisingly as structured pruning is not as commonly used in FCNs \cite{oneillOverviewNeuralNetwork2020}. More successful attempts could be due to alternative scoring methods \cite{suzuki2018spectral}; our approach might require wider layers to conform with theoretical conditions.  

\textbf{Tail index and generalization.} Following our theoretical results, we examine whether the heavier-tailed networks lead to better generalization performance. Consistent with our hypothesis, Figure \ref{fig:tail_index_gen} shows that models with the highest tail index have consistently the worst test performances. The same conclusion applies to generalization performance as training accuracy is $100\%$ for all models.
\label{sec:tail-index-gen}
\begin{figure}[h]
    \centering
    \includegraphics[width=\textwidth]{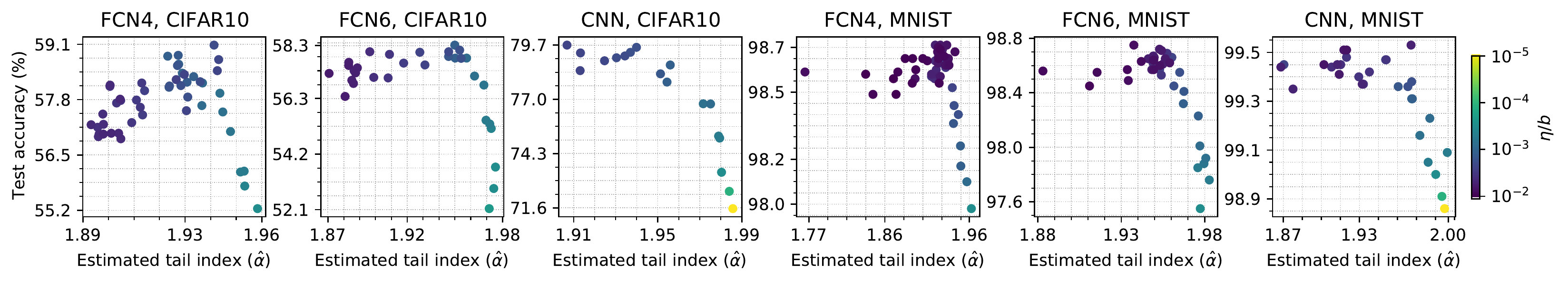}
    \caption{Test accuracy vs. mean estimated tail index ($\hat{\alpha}$) for each model. Color: training $\eta/b$ ratio.}
    \label{fig:tail_index_gen}
\end{figure}

\section{Conclusion}
\label{sec:conclusion}
We investigated the conditions and the mechanism through which various pruning strategies can be expected to work, and confirmed the relationship between pruning and generalization from our theoretical approach. Future directions (hence our limitations) include a formal proof for the convergence of the weights to a mean-field distribution in the heavy-tailed setting, eliciting the relationship between $\hat{\alpha}$ and $\kappa$ more clearly, and further examination of structured pruning in FCNs. Other pruning strategies would also benefit from our analyses, e.g. gradient-based pruning \cite{blalock2020state}. Recently it has been observed  that gradients also exhibit heavy-tailed behavior \cite{csimcsekli2019heavy,zhang2020adaptive,zhou2020towards}; we suspect that our theory might be applicable in this case as well. 

\section*{Acknowledgments}
The authors are grateful to A. Taylan Cemgil, Valentin De Bortoli, and Mohsen Rezapour for fruitful discussions on various aspects of the paper. Umut \c{S}im\c{s}ekli's research is supported by the French government under management of Agence Nationale de la Recherche as part of the ``Investissements d’avenir'' program, reference ANR-19-P3IA-0001 (PRAIRIE 3IA Institute).

\bibliography{generalization}

\newpage
\begin{appendices}

\begin{center}

\Large \bf Appendix

\end{center}

The appendix is organized as follows:
\begin{itemize}
    \item Section~\ref{sec:sup_exp} details the experimental setup used in our simulations and presents some additional experimental results.
     \item In Section~\ref{sec:sup_genOriginal}, a generalization bound for an uncompressed network, given that this network is compressible, is presented.
     \item  Section~\ref{sec:sup_perturbation} provides an upper-bound on the change in the network output when there is a small change in the network weights.
     \item In Section~\ref{sec:sup_compressionRate}, the relation between compressibility and the tail-index is discussed.
     \item Proofs of the main results of the paper are presented in Section~\ref{sec:sup_mainProofs}.
     \item Finally, the technical lemmas are proved in Section~\ref{sec:sup_lemProofs}.
\end{itemize}

\section{Details of the Experiments and Additional Results} 
\label{sec:sup_exp}
Here we provide a more detailed explanation for our experimental setting, as well as the results we omitted from the main paper due to space restrictions. You can find the source code for the experiments at the following link: \href{https://github.com/mbarsbey/sgd_comp_gen}{https://github.com/mbarsbey/sgd\_comp\_gen}.

\subsection{Datasets}
The experiments were conducted in a supervised learning setting where the task is classification of images. Each model is trained on CIFAR10 \cite{cifar102009} and MNIST \cite{mnist2010} datasets. The MNIST is an image classification dataset where the data is comprised of $28\times 28$ black and white handwritten digits, belonging to one class from $0$ to $9$. We use the traditional split defined in the dataset where there are $60000$ training and $10000$ test samples. CIFAR10 is also image classification dataset comprising  $32\times 32$ color images of objects or animals, making up 10 classes. There are $50000$ training and $10000$ test images, which is the split that we use in the experiments.

\subsection{Models}

As described in the main text, in our experiments we use three models: a fully connected network with 4 hidden layers (FCN4), a fully connected network with 6 hidden layers (FCN6), and a convolutional neural network with 8 convolutional layers (CNN). Hidden layer widths are $2048$ for the two FCN models. All networks include ReLU activation functions and none of them include batch normalization, dropout, residual layers, or any explicit regularization term in the loss function. The architecture for the CNN model for the CIFAR10 dataset progresses as below:

$$
64, M, 128, M, 256, 256, M, 512, 512, M, 512, 512, M,
$$

where integers stand for 2-dimensional convolutional layers (and the corresponding number of filters) with a kernel size of $3 \times 3$, and $M$ stands for $2 \times 2$ max-pooling with a stride of $2$. Our CNN architecture follows that of VGG11 model \cite{simonyanVeryDeepConvolutional2015} except after the layers presented above we have only a single linear layer with a softmax output. For the MNIST experiment the first max-pooling layer was omitted as the dimensions of the MNIST images disallow the previous structure to be used. Table \ref{tab:parameter_count} includes the number of parameters for each model-dataset combination.
\begin{table}[h]
\centering
\begin{tabular}{ |c| c| c| c|}
\hline
 & FCN4 & FCN6 & CNN \\ 
 \hline
 CIFAR10 & 18,894,848 & 27,283,456 & 9,222,848\\
 \hline
 MNIST & 14,209,024 & 22,597,632 & 9,221,696 \\
 \hline
\end{tabular}
\caption{\label{tab:parameter_count}Parameter counts for all model-dataset combinations.}
\end{table}

\subsection{Training and hyperparameters}
All models were trained with SGD until $100\%$ training accuracy and a training negative log-likelihood less than $5 \times 10^{-5}$ is acquired, with constant learning rates and no momentum. The training hyperparameters include two batch-sizes ($b = 50, 100$) and a variety of learning rates ($\eta$) to generate a large range of $\eta/b$ values. The ranges vary somewhat since different $\eta/b$ values might lead to heavy-tailed behavior (or divergence) under different settings. Table \ref{tab:eta_b_ranges} presents these ranges for all experiments. See the source code for all experiment settings that were presented in the main results.

\begin{table}[h]
\centering
\begin{tabular}{ |c| c| c| c|}
\hline
 & FCN4 & FCN6 & CNN \\ 
 \hline
 CIFAR10 & $5 \times 10^{-5}$ to  $2.7\times 10^{-3}$  & $2.5 \times 10^{-5}$ to  $4\times 10^{-3}$ & $1 \times 10^{-5}$ to  $1.5\times 10^{-3}$\\
 \hline
 MNIST & $5 \times 10^{-5}$ to  $1.14\times 10^{-2}$  & $5 \times 10^{-5}$ to  $8.8\times 10^{-3}$ & $1 \times 10^{-5}$ to  $6.35\times 10^{-3}$ \\
 \hline
\end{tabular}
\caption{\label{tab:eta_b_ranges}$\eta/b$ ranges for all experiments.}
\end{table}

\subsection{Tail index estimation}
We use the following multivariate tail-index estimator proposed by \cite{bgindex}.

\begin{theorem}[{\cite[Corollary 2.4]{bgindex}}]
Let $\{X_i\}_{i=1}^K$ be a collection of i.i.d.\ random vectors where each $X_i$ is multivariate strictly stable with tail index $\alpha$, and $K = K_1 \times K_2$.
Define $Y_i \triangleq \sum_{j=1}^{K_1} X_{j+(i-1)K_1} \>$ for $i \in \llbracket 1, K_2 \rrbracket$. Then, the estimator
\begin{align}
\label{eqn:alpha_estim}
\widehat{\phantom{a}\frac1{\alpha}\phantom{a}} \hspace{-4pt} \triangleq \hspace{-2pt} \frac1{\log K_1} \Bigl(\frac1{K_2 } \sum_{i=1}^{K_2} \log \|Y_i\|  - \frac1{K} \sum_{i=1}^K \log \|X_i\| \Bigr). 
\end{align}
converges to $1/{\alpha}$ almost surely, as $K_2 \rightarrow \infty$.
\end{theorem}

This estimator has been used in previous research such as \cite{csimcsekli2019heavy} and \cite{tzagkarakisCompressiveSensing2018}. We center the observations using the median values before the estimation. Using the alternative univariate tail-index estimator from the same paper \cite[Corollary 2.2]{bgindex} has no qualitative effects on our results. An additional benefit of our choice is additional analyses it makes possible, as presented in Section \ref{sec:synthetic}. Comparisons with alternative tail-index estimators with symmetric $\alpha$-stable assumption revealed no dramatic differences between various estimators \cite{satheEstimationParameters2020}.

\subsection{Pruning details}
We first provide a review of the pruning methods we use. All three notions of pruning in our experiments correspond to the magnitude-wise ordering of certain parameters and the `pruning' of a certain ratio of parameters that correspond to smallest magnitudes\footnote{Note that a more relaxed definition of pruning would be `systematic removal of model parameters' to allow for different scoring methods in pruning \cite{blalock2020state}. However, we proceed with our specific definition since this allows us to communicate our theoretical and experimental results more concisely.}. When the parameters that are pruned are the weight parameters themselves, this corresponds to magnitude-based pruning or \textit{magnitude pruning} as known in the literature, which can be conducted layer-wise or globally \cite{blalock2020state}. \textit{Singular value pruning}, as described here, corresponds to pruning of the smallest singular values (and, by implication, the related singular vectors) in the SVD of specific layers. To apply the SVD to CNN layers, we reshape the parameter tensors into matrices of shape (\# channels) $\times$ (\# filters $\times 3 \times 3$). Lastly \textit{node pruning} corresponds to the pruning of the whole columns in the weight matrices. Again, CNN counterpart of node pruning is open to interpretation; we choose to prune specific kernels according to the their norms.

Before any pruning is done, the parameters to be pruned are centered with the estimated median of the observations, in order to conform with our tail index estimation methodology. We chose median due to its robustness against extreme observations especially with a small sample - however our results were qualitatively unchanged when the mean was used in the centering. After the pruning (in all three methods), the median was added to the pruned parameter vectors before testing the performance of the resulting model. Note that the median (or mean) was usually very small in norm and omitting centering made no qualitative effect on the results. The last layer in all models and additionally the first layer in VGG models were not pruned as they had very small numbers of parameters compared to the other layers.

Lastly, while `remaining parameter ratio' ($\kappa$) or `pruning ratio' ($1- \kappa$) are easy to interpret in the case of magnitude pruning or node pruning, in SVD $\kappa$ would equal (number of singular value and vector parameters left) / (number of weight parameters in the layer), and pruning ratio would be determined accordingly.

\subsection{Additional results}
Here we present additional results that were referenced in the main text but were not presented due to space restrictions.
\subsubsection{Synthetic experiments}
\label{sec:synthetic}
Recall that in Figure \ref{fig:alpha_kappa} we examined whether lower tail index is associated with higher prunability, for a given relative compression error. The results demonstrated that this was indeed the case. Here we compare our empirical results with some synthetic results to get additional insights regarding whether HML property is actually observed in our networks.

For this experiment, we randomly sample tail indices $\alpha_i \sim \mathcal{U}(1.75, 2),$ where $i \in \{1, \dots, 250\}$. Then for each $\alpha_i$, we sample three different `weight matrices': $\mathbf{W}_{ind, i}, \mathbf{W}_{col, i}, \mathbf{W}_{lay, i} \in \mathbb{R}^{500 \times 500}$. The elements of $\mathbf{W}_{ind, i}$ are sampled independently from a $\sas$ with tail index $\alpha_i$; this corresponds to the case where weight parameters are statistically independent as prescribed by the HML property. On the other hand, columns of the $\mathbf{W}_{col, i}$ are independently sampled from a $500$-dimensional multivariate \textit{elliptically contoured} $\sas$ with tail index $\alpha_i$. A $d$-dimensional elliptically contoured multivariate $\sas$ has the characteristic function

$$
\mathbb{E}\left[\exp(i\langle \omega, X \rangle)\right] = \exp(- \|\omega\|^{\alpha}),
$$

where $X, \omega \in \mathbb{R}^d$ and $\langle \cdot , \cdot \rangle$ stands for inner product. This means that while the columns of the matrix are independent, column elements can be correlated. Lastly, all elements of $\mathbf{W}_{lay, i}$ are sampled from a $(500 \times 500)$-dimensional elliptically contoured multivariate $\sas$, creating a case where all elements of the matrix can be correlated. We repeat the analysis presented in Figure \ref{fig:alpha_kappa} for all three sets of sampled synthetic layer weights, and present the results in Figure \ref{fig:synth_alpha_pruning_ratio}.

\begin{figure}[h]
    \centering
    \includegraphics[width=.60\textwidth]{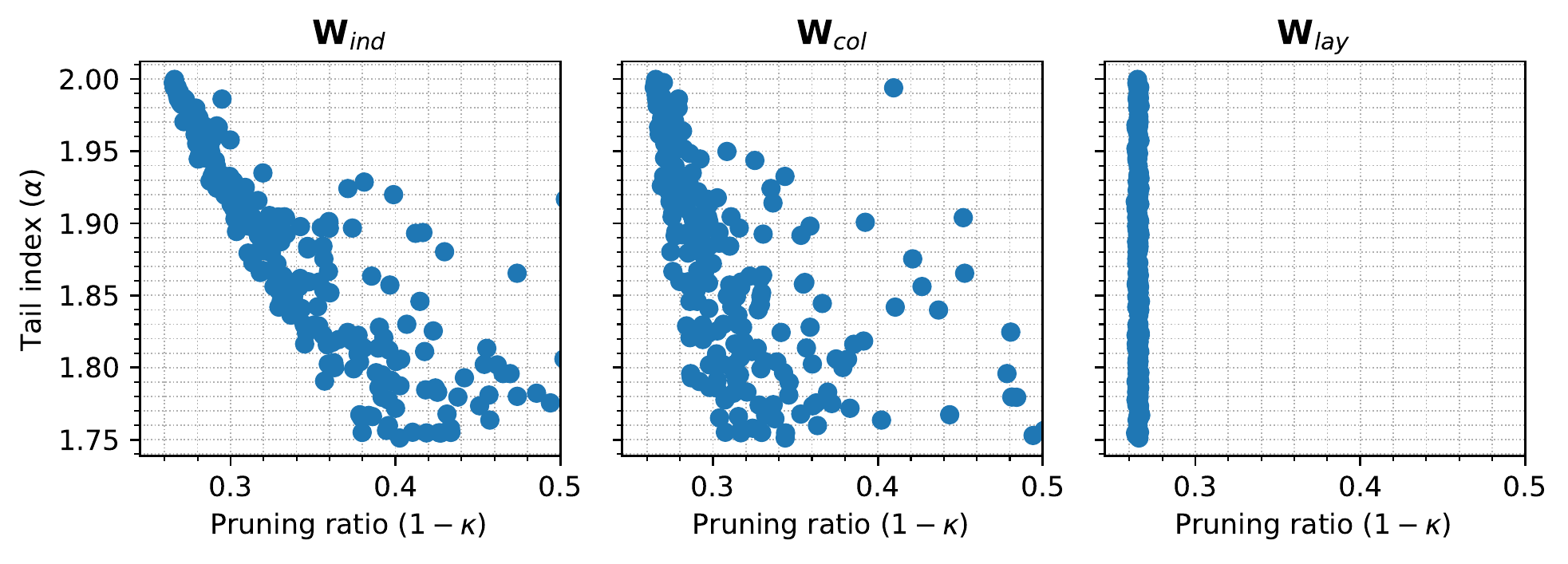}
    \caption{Tail index ($\alpha$) vs. pruning ratio ($1 - \kappa$), with relative compression error $= 0.1$, for synthetically generated weight matrices.}
    \label{fig:synth_alpha_pruning_ratio}
\end{figure}

The results demonstrate two interesting phenomena. First, a comparison of these results with Figure \ref{fig:alpha_kappa} shows that our empirical results show the most similarity with results obtained with $\mathbf{W}_{ind, i}$, showing support for the existence of the HML property. Another observation is that as the layer parameters become correlated, the prunability advantage conferred by heavier tails disappears. This observation both supports the existence of HML property in overparametrized neural networks and invites further discussion on the importance of propagation of chaos phenomenon in such architectures for compressibility and generalization \cite{de2020quantitative}.

\subsubsection{Global magnitude pruning and node pruning results}
In Figure \ref{fig:global_magnitude} present the global magnitude pruning results for magnitude pruning. The results are qualitatively very similar to those of the layer-wise magnitude pruning. Figure \ref{fig:node_pruning_fcn} presents the results of node pruning on FCNs. As mentioned in the main text, the less impressive results might have to do with the layer widths not being sufficient for our theoretical conditions.

\begin{figure}[h]
    \centering
    \includegraphics[width=\textwidth]{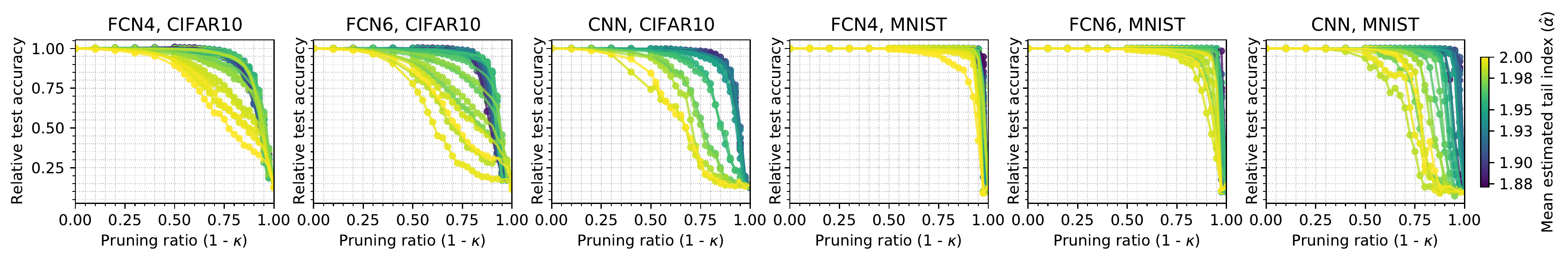}
    \caption{Relative test accuracy vs. pruning ratio for global magnitude pruning. Color: mean $\hat{\alpha}$.}
    \label{fig:global_magnitude}
\end{figure}

\begin{figure}[h]
    \centering
    \includegraphics[width=.67\textwidth]{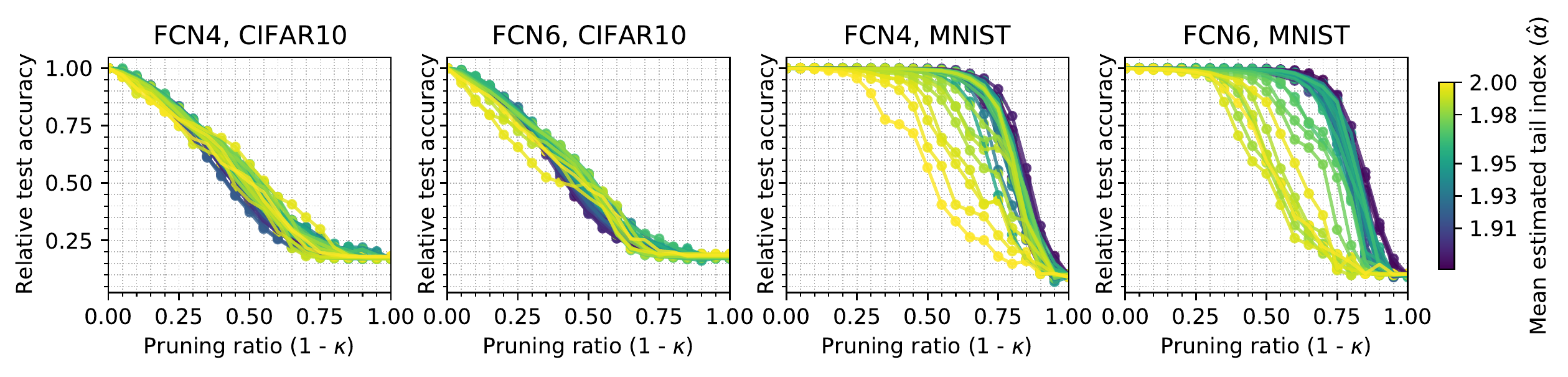}
    \caption{Relative test accuracy vs. pruning ratio for node pruning in FCNs. Color: mean $\hat{\alpha}$.}
    \label{fig:node_pruning_fcn}
\end{figure}

\subsection{Hardware and other resources}

The experiments were conducted on an internal server of a research institute. Nvidia Titan X, 1080 Ti, and 1080 model GPU's were used roughly equally in running the experiments. Our published results involve 215 trained models, each of which included GPU-heavy workload, with an average completion time of approximately 2 hours. Around 30 models diverged (thus were not used in the results) and in most cases were trained for less than an hour. Total GPU-time could correspondingly be estimated to equal 460 hours.  We also conducted tail index estimation and pruning experiments on these networks, however the computational load of these experiments are negligible compared to those of the training of the algorithms, with an estimated CPU time of 48 hours for all estimation and pruning tasks that were published.

As can be seen in the accompanying code, the experiments were conducted using Python programming language. The deep learning framework PyTorch \cite{pytorch2019} as well as some of its tutorials\footnote{\textsc{https://github.com/pytorch/vision/blob/master/torchvision/models/vgg.py}} were used in implementing the experiments.

\section{Generalization bound for the uncompressed network} \label{sec:sup_genOriginal}

In this section, in the continuation of Section~\ref{sec:generalization}, we establish a generalization bound for an  uncompressed network,  if this network is compressible using the layer-wise magnitude pruning strategy. 

\begin{theorem} \label{th:generalizationOriginal}
Assume \Cref{asmp:magp} holds. Then for $n \colon n/\log(n) \geq \max(9L,81\varepsilon^{-2\kappa})$, $\{d_l\}_{l=1}^L \colon d_l \geq d_{l,0}$ and   $d \geq 10$, and any $ \delta,\tau>0 $, with probability  at least $1-2e^{-d/2}-\delta-\epsilon$,
\begin{align}
\Rcal_{0}({\vc{{w}}}) & \leq \widehat{\Rcal}_{\tau}({\vc{w}})+ \max\left(2,24 \rho_{\varepsilon}(\kappa,d) \mathcal{L}(\tau,\delta)R(\delta)/\sqrt{d}\right)\sqrt{d\log(n)/n}. \label{eq:generalizationOriginal}
\end{align}
where $R(\delta)$ and $\mathcal{L}(\tau,\delta)$ are defined in \eqref{def:RofDeltaLipchtz}, $\kappa\coloneqq \frac{1}{d}\sum_{l=1}^L \lceil \kappa_l d_l\rceil$,
\begin{align*}
    \rho_{\varepsilon}(\kappa,d)&\coloneqq \min \left(\varepsilon^{1-\kappa}\exp \left(h_b(\kappa)+h_b^{(1)}(\kappa,d)\right),1\right) \leq  \min \left(3\varepsilon^{1-\kappa} ,1\right),\\
    h_b^{(1)}(\kappa,d)&\coloneqq \frac{\lceil d/2\rceil}{d} \max\left(h_b(\lceil \kappa d/2 \rceil/\lceil d/2\rceil),h_b(\lfloor \kappa d/2\rfloor/\lceil d/2\rceil)\right).
\end{align*}
\end{theorem}
Note that the function $h_b^{(1)}(\kappa,d)\leq \log(2)(1/2+1/d)$. Hence
\begin{align*}
    \varepsilon^{1-\kappa} e^{h_b(\kappa)+h_b^{(1)}(\kappa,d)}  \leq   \varepsilon^{1-\kappa} e^{log(2)(3/2+1/d)} \leq 3 \varepsilon^{1-\kappa},
\end{align*} for $d\geq 12$. Besides,   $\rho_{\varepsilon}(1,d)=1$  and $\rho_{\varepsilon}(0,d)=\varepsilon$. Moreover, when both $d$ and $\kappa d$ are even 
\begin{figure}[t]
 \centering
 \begin{tabular}{c c c}
   \includegraphics[width=0.3\textwidth]{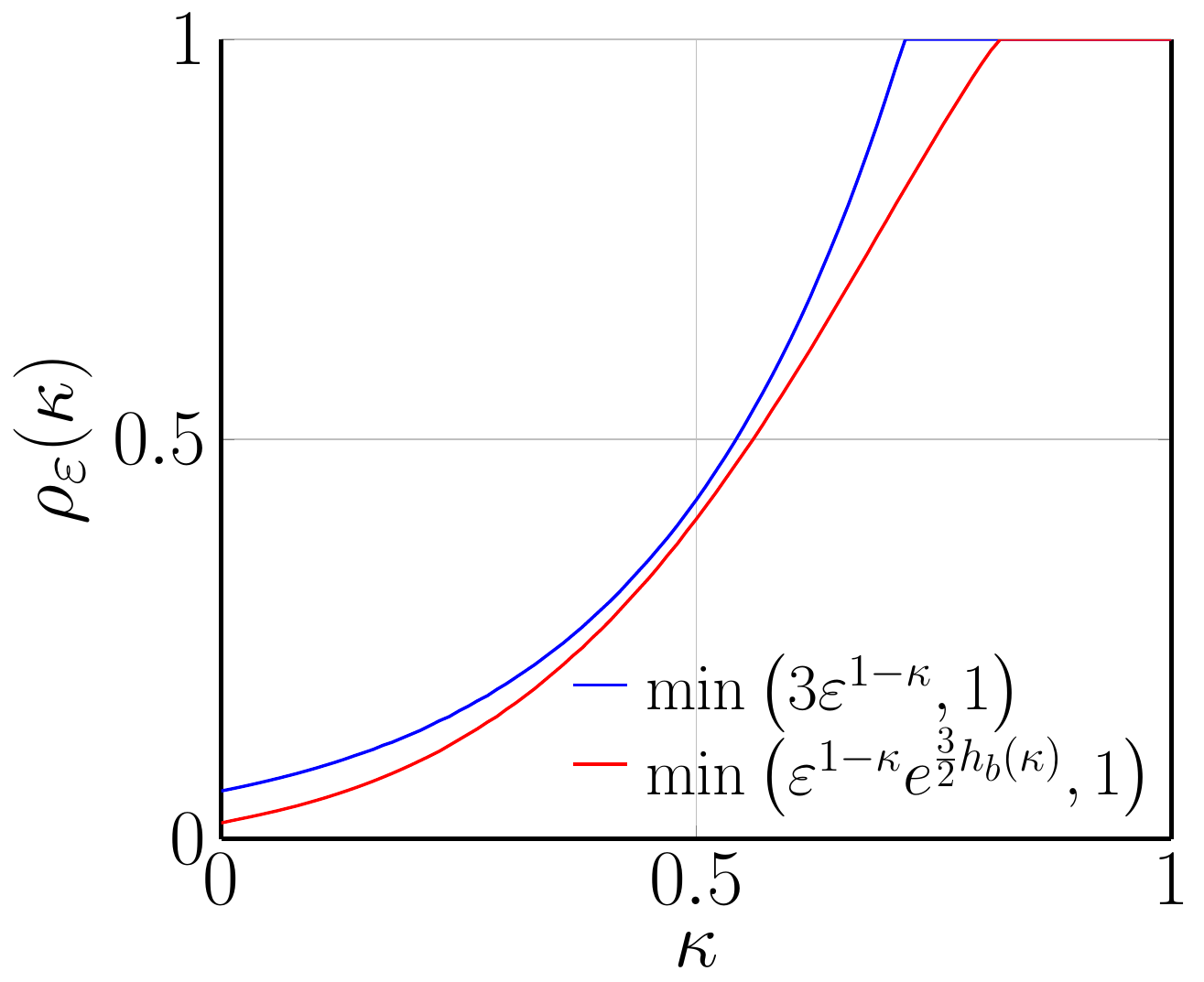}  &  \includegraphics[width=0.3\textwidth]{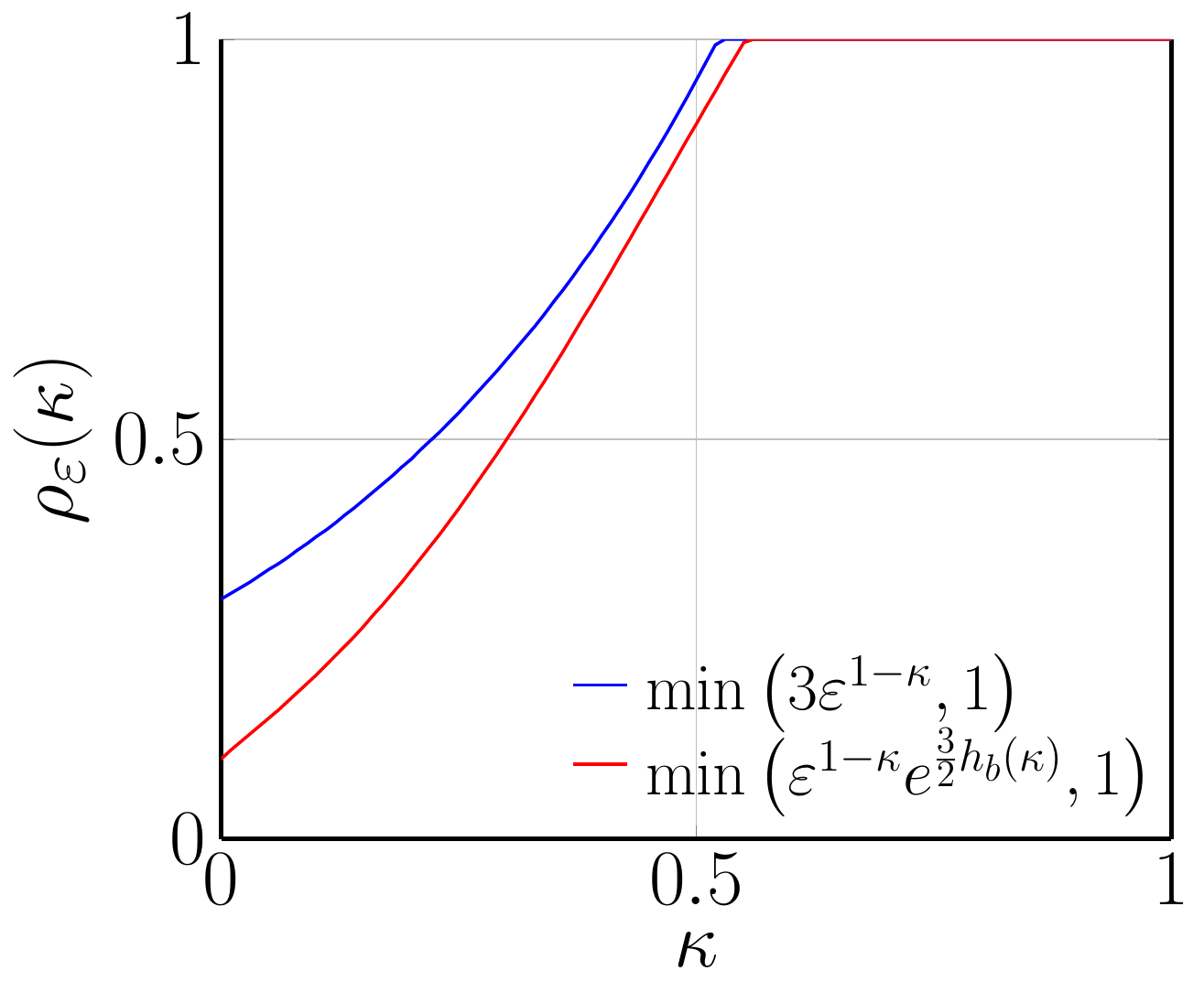}   &  \includegraphics[width=0.3\textwidth]{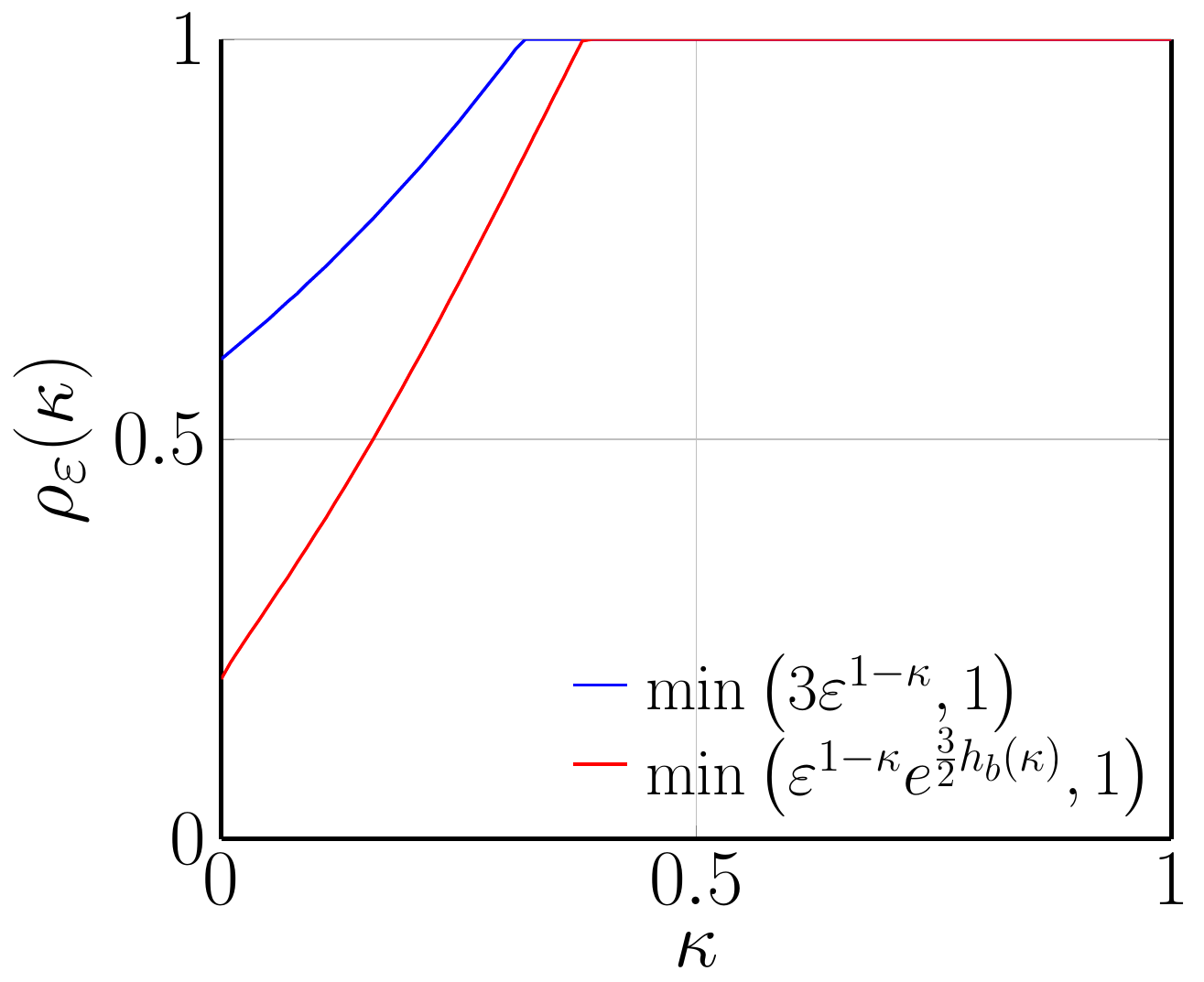}\\
   $\varepsilon=0.02$ &  $\varepsilon=0.1$ & $\varepsilon=0.2$
 \end{tabular}
 \caption{$\rho_{\varepsilon}(\kappa)$ for different values of $\varepsilon$.}              \label{fig:rho}
\end{figure}
numbers, then $ h_b^{(1)}(\kappa,d)=\frac{1}{2}h_b(\kappa)$ and $ \rho_{\varepsilon}(\kappa)\coloneqq \rho_{\varepsilon}(\kappa,d)$ is increasing with respect to $\kappa$. To show the latter claim, consider the derivative of $g\left(\kappa\right) \coloneqq \varepsilon^{1-\kappa}\exp \left(\frac{3}{2}h_b(\kappa)\right)$ with respect to $\kappa$. This derivative is equal to zero at $\kappa^*=1/\left(1+\varepsilon^{2/3}\right)$ and is positive for $\kappa <\kappa^*$. In addition, $g\left(\kappa^*\right)=\left(1+\varepsilon^{2/3}\right)^{3/2}>1$. Hence, $\rho_{\varepsilon}(\kappa)=\min\left(g(\kappa),1\right)$ is increasing with respect to $\kappa$. In Figure~\ref{fig:rho}, $\rho_{\varepsilon}(\kappa)=\min\left( \varepsilon^{1-\kappa}\exp \left(\frac{3}{2}h_b(\kappa)\right),1\right)$, together with its upper bound $\min\left( 3 \varepsilon^{1-\kappa},1\right)$ are plotted for different values of $\varepsilon$. 

It can be observed that if a network is more compressible, then not only the compressed network, but also the original network has a better generalization bound. This result is consistent with the findings of \cite{hsu2021generalization,kuhn2021robustness}. In  \cite{hsu2021generalization}, it is shown that if two networks are ``\textit{close}'' enough, a good generalization bound on one of them, would imply a good bound on the other one as well. In \cite{kuhn2021robustness}, it is shown that ``\textit{prunability}'' of a network captures well the generalization property of the network.

Finally, it should be noted that when the weights of the network follow a stable distribution, similar results to Corollary~\ref{cor:genCompressStable} can be established for the original network.

\section{Perturbation Bound}
\label{sec:sup_perturbation}
The goal of pruning is to find  compressed weights $\vc{\widehat{w}}$ with low dimensionality that are close enough to the original weights $\vc{w}$, which is measured in this work by the relative error  $\|\vc{\widehat{w}}-\vc{w}\|/\|\vc{w}\|$. The following perturbation result guarantees that such pruning strategies also result in small perturbations on the output of the network. The proof is based on the technique given in \cite{neyshabur2017pac}.

\begin{lem} \label{lem:prunOutputDiff}
Let $\vc{w},\widehat{\vc{w}} \in \rset^d$ be two fully connected neural networks. Assume that there exists $\{ \veps_l \}_{l=1}^L\colon \veps_l \geq 0$, such that $\| \widehat{\vc{w}}_l-\vc{w}_l\| \leq \veps_l \|\vc{w}_l\|$, for all $l \in \llbracket 1, L\rrbracket$.
Then, the following inequality holds:
\begin{align}
\|f_{\widehat{\vc{w}}}(\vc{x}) - f_{{\vc{w}}}(\vc{x})  \| \leq B \left[ \prod\nolimits_{l=1}^{L} \left( 1 + \veps_l \right) -1 \right] \left[ \prod\nolimits_{l=1}^{L}  \|\vc{w}_l\| \right], \label{eq:netOutputDiff}
\end{align}
for all $\vc{x} \in \mathcal{X}_B$. In particular, if $\varepsilon_l=\varepsilon$ for all layers and $\| \vc{w} \| \leq R$, then 
\begin{align}
      \|f_{\widehat{\vc{w}}}(\vc{x}) - f_{{\vc{w}}}(\vc{x})  \|   &\leq  B \left[\left( 1 + \veps \right)^L -1 \right] \left({R}/{\sqrt{L}}\right)^{L}. \label{eq:netOutputDiffParticular}
    \end{align}
\end{lem}
For derivation of the above bound on the network outputs, the worst case in the propagation of the errors of each layer is assumed, which results into an exponential dependence on the depth of the network, similarly to \cite{neyshabur2017pac}. 

\section{Compression rate and tail-index} \label{sec:sup_compressionRate}
In this work, different pruning strategies have been investigated by exploiting the compressibility properties of heavy-tail distributions. In this section, we show that moreover, in some certain sense, heavier-tailed distributions are more compressible. However, we must underline that this result is neither comprehensive, nor directly usable in our framework, as we will discuss after stating it. 

Before stating the result, let us define the following quantity. For $\varepsilon>0$ and $\vc{w} 
 \in \mathbb{R}^d$, let
\begin{align}
   \kappa_p(\vc{w},\varepsilon)\coloneqq \min \left\{\kappa \colon \left(\|\mathbf{w}^{(\kappa d)} - \mathbf{w}\|_p / \|\mathbf{w}\|_p\right)  \leq \varepsilon \right\}.
\end{align}

\begin{proposition} \label{prop:compressionRate} Suppose that $\vc{w}_1\in \mathbb{R}^d$ and $\vc{w}_2\in \mathbb{R}^d$ are independent vectors of  i.i.d. heavy-tailed random variables with  tail indices $\alpha_1$ and $\alpha_2$, respectively. If $\alpha_1 > \alpha_2$, then for any $\kappa,\varepsilon,\delta>0$ and $p< \max(\alpha_1,\alpha_2)$, there exists 
$d_0(\delta)$, such that for $d\geq d_0(\delta)$, 
\begin{align}
    \mathbb{E}\left[\|\mathbf{w}_1^{(\kappa d)} - \mathbf{w}_1\|_p / \|\mathbf{w}_1\|_p\right]+\delta>\mathbb{E}\left[\|\mathbf{w}_2^{(\kappa d)} - \mathbf{w}_2\|_p / \|\mathbf{w}_2\|_p\right], \label{eq:compresExpEr}
\end{align}
and 
\begin{align}
    \mathbb{E}\left[\kappa_p\left(\vc{w}_1,\varepsilon \right)\right]+\delta > \mathbb{E}\left[\kappa_p\left(\vc{w}_2,\varepsilon \right)\right]. \label{eq:compresExpRate}
\end{align}
\end{proposition}

The above proposition shows that for a fixed $p$-norm of the normalized compression error with $p<\alpha$, the heavier-tailed distributions are more compressible. The caveat here is that for $p\geq \max(\alpha_1,\alpha_2)$, all terms in \eqref{eq:compresExpEr} and \eqref{eq:compresExpRate} go to zero due to Lemma~\ref{lem:compress} and hence \eqref{eq:compresExpEr}  and  \eqref{eq:compresExpRate}  trivially hold. Therefore, unfortunately we cannot use Proposition~\ref{prop:compressionRate} in our framework since we are mainly interested in the case where $p\geq \max(\alpha_1,\alpha_2)$. Investigating the level of compressibility as a function of the tail-index is a natural next step for our study.

\section{Proofs of the Main Results} \label{sec:sup_mainProofs}
In this section we provide proofs of our main results. We shall begin with stating the following result from \cite{gribonval2012compressible}, which will be repeatedly used in our proofs.

\subsection{Existing Theoretical Results} \label{sec:knownResults}
Many of our results are based on the compressibility of  i.i.d. instances of heavy-tail random variables. Here, we state a known result regarding this fact.
\begin{lem}[{\cite[Proposition 1, Part 2]{gribonval2012compressible} } ]
\label{lem:compress} 
Let $x\in \rset$ be a random variable and assume that $\mathbb{E}|x|^\alpha = \infty$ for some $\alpha \in \rset_+$. Then for all $p \geq \alpha$, $0 < \kappa \leq 1$ and any sequence $\kappa_d$ such that $\lim_{d\to \infty} \frac{\kappa_d}{d} = \kappa$, the following identity holds almost surely:
\begin{align}
\lim_{d \to \infty} \left( \|\vc{x}^{(\kappa_d)} - \vc{x}\|_p/\|\vc{x}\|_p  \right) = 0,
\end{align}
where $\vc{x}  = (x_1, \dots, x_d)$ is a vector of i.i.d. random variables of length $d$.
\end{lem}
A similar result was concurrently proven in \cite{amini2011compressibility}. As stated before, this result informally states that for a large vector of i.i.d.\ heavy-tailed random variables, the norm of the vector is mainly determined by a small fraction of its entries. To show this visually, we have generated $10^4$ i.i.d. random variables $\{x_i\}_i$ with $\sas(1)$ distribution where $\alpha=1.7$. Then, we have plotted the histogram of $|x|$ in Figure~\ref{fig:histogramStable}. As can be seen in the figure, the norm of the whole vector is mainly dominated by few number of samples.
\begin{figure}[t]
 \centering
   \includegraphics[width=0.9\textwidth]{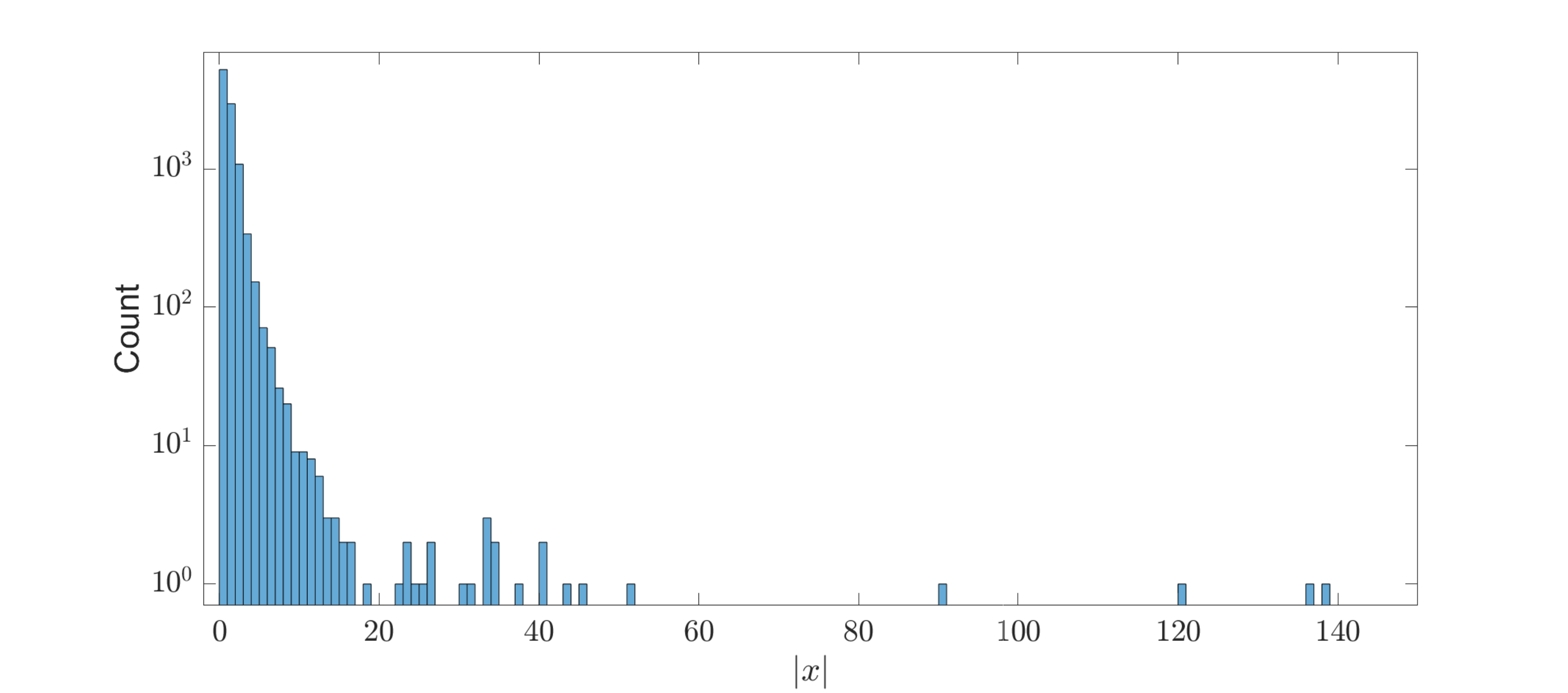}   
 \caption{Histogram of $|x|$ for a sequence of i.i.d. random variables distributed according to $\sas(1)$, where $\alpha=1.7$.}              \label{fig:histogramStable}
\end{figure}

\subsection{Proof of Theorem~\ref{th:magPruning}}

\begin{proof}
\begin{enumerate}[label=(\roman*),noitemsep,topsep=0pt,leftmargin=*,align=left]
    \item As $\{d_l\}_{l=1}^L$ grow, due to HML property and assumptions of this part of the theorem, $\vc{w}$ converges in distribution to a heavy-tailed random vector, denoted as $\vc{x}=(x_1,\ldots,x_d)\in \mathbb{R}^d$,  with i.i.d.\ elements and tail index $\alpha \in (0,2)$.  Hence, for any $\epsilon>0$, $\varepsilon>0$, $\kappa \in (0,1)$, and $p \geq \alpha$ there exists $\{d'_{l,0}\}_{l=1}^L\colon d'_{l,0} \in \mathbb{N}$ such that for $d_l \geq d'_{l,0}$, $l \in \llbracket 1,L \rrbracket$,
\begin{align}
    \mathbb{P}\left[\left(\|\vc{x}^{(\kappa d)}-\vc{x}\|_p/\|\vc{x}\|_p\right)\leq \varepsilon\right] - \mathbb{P}\left[\left(\|\vc{w}^{(\kappa d)}-\vc{w}\|_p/\|\vc{w}\|_p\right)\leq \varepsilon\right]  \leq \epsilon. \label{eq:magproof1}
\end{align}

Moreover due to Lemma~\ref{lem:compress}, there exists $d''_{0}\in \mathbb{N}$ such that for $d \geq d''_{0}$, 
\begin{align}
    \mathbb{P}\left[\left(\|\vc{x}^{(\kappa d)}-\vc{x}\|_p/\|\vc{x}\|_p\right)\leq \varepsilon\right]=1 \label{eq:magproof2}
\end{align}
The results follows from \eqref{eq:magproof1} and \eqref{eq:magproof2} and by choosing $d_{l,0} \geq d'_{l,0}$, for $l \in \llbracket 1,L \rrbracket$, such that  $\sum_{l=1}^L d_{l,0} \geq d''_{0}$.

\item The proof is similar to the previous part. As $\{d_l\}_{l=1}^L$ grow, due to HML property, for $l \in \llbracket 1, L \rrbracket$, $\vc{w}_l$ converges in distribution to a heavy-tailed random vector, denoted as $\vc{x}_l=(x_{l,1},\ldots,x_{l,d_l})\in \mathbb{R}^{d_l}$ ,  with i.i.d.\ elements and tail-index $\alpha_l \in (0,2)$. Hence, for any $\epsilon>0$, $\{\varepsilon_l\}_{l=1}^L\colon \varepsilon_l>0$, $\{\kappa_l\}_{l=1}^L\colon \kappa_l \in (0,1)$,  and $p \geq \max_l \alpha_l$ there exists $\{d'_{l,0}\}_{l=1}^L\colon d'_{l,0} \in \mathbb{N}$ such that for $d_l \geq d'_{l,0}$, $l \in \llbracket 1,L \rrbracket$,
\begin{align}
    \mathbb{P}\left[\left(\|\vc{x}_l^{(\kappa_l d_l)}-\vc{x}_l\|_p/\|\vc{x}_l\|_p\right)\leq \varepsilon_l\right] - \mathbb{P}\left[\left(\|\vc{w}_l^{(\kappa_l d_l)}-\vc{w}_l\|_p/\|\vc{w}_l\|_p\right)\leq \varepsilon_l\right] \leq \epsilon. \label{eq:magproof3}
\end{align}

Moreover due to Lemma~\ref{lem:compress}, there exists $\{d''_{l,0}\}_{l=1}^L\colon d''_{l,0} \in \mathbb{N}$ such that for $d_l \geq d''_{l,0}$, 
\begin{align}
    \mathbb{P}\left[\left(\|\vc{x}_l^{(\kappa_l d_l)}-\vc{x}_l\|_p/\|\vc{x}_l\|_p\right)\leq \varepsilon_l\right]=1 \label{eq:magproof4}
\end{align}
The results follows from \eqref{eq:magproof4} and \eqref{eq:magproof3} and by choosing $d_{l,0} \geq \max(d'_{l,0},d''_{l,0})$.
\end{enumerate} \end{proof}

\subsection{Proof of Theorem~\ref{th:spectral}}
\begin{proof}
Fix $l \in \llbracket 2, L-1 \rrbracket$ and recall that $\vc{W}_l \in \rset^{h_l \times h_{l-1}}$ with $d_l = h_l \times h_{l-1}$. Define
\begin{align}
    \vc{X}_l := \frac1{h_{l}^{2/\alpha_l}} \vc{W}_l^\top \vc{W}_l
\end{align}
and denote the eigenvalues of $\vc{X}_l$ by $\boldsymbol{\lambda}_l = [\lambda_{l,1}, \dots, \lambda_{l,h_{l-1}}]$.

Let $\vc{U}_l \in \rset^{h_l \times h_{l-1}}$ be a matrix whose entries are independent and identically distributed from a symmetric stable distribution with tail-index $\alpha_l$. Note that $\vc{W}_l$ converges in distribution to $\vc{U}_l$, as network dimension goes to infinity, due to HML property and the assumptions of the theorem. Similarly, define
\begin{align}
    \vc{X}'_l := \frac1{h_{l}^{2/\alpha_l}} \vc{U}_l^\top \vc{U}_l
\end{align}
and denote the eigenvalues of $\vc{X}'_l$ by $\boldsymbol{\lambda}'_l = [\lambda'_{l,1}, \dots, \lambda'_{l,h_{l-1}}]$.

As $\vc{W}_l$ converges in distribution to $\vc{U}_l$, then $\lambda_{l,k}$ also weakly converges to $\lambda'_{l,k}$, due to Weyl's inequality (\cite[Page~63]{bhatia97}). Hence, for any $\epsilon>0$,  $\{\varepsilon_l\}_{l=2}^{L-1}\colon \varepsilon_l>0$, and $\{\kappa_l\}_{l=2}^{L-1}\colon \kappa_l \in(0,1)$,  there exists $\{\hat{h}_{l,0}\}_{l=1}^{L-1}$, such that for every $l \in \llbracket 2, L-1\rrbracket$, and $h_i \geq \hat{h}_{i,0}$ for all $i \in \llbracket 1, l\rrbracket$, the following holds
\begin{align}
 \left|\mathbb{P} \left[ \frac{\|\boldsymbol{\lambda}_i^{(\kappa_i h_{i-1})} - \boldsymbol{\lambda}_i  \|_p}{\|\boldsymbol{\lambda}_i\|_p} \leq \veps_i^2 \right] -\mathbb{P} \left[ \frac{\|\boldsymbol{\lambda}_i^{'(\kappa_i h_{i-1})} - \boldsymbol{\lambda}'_i  \|_p}{\|\boldsymbol{\lambda}'_i\|_p} \leq \veps_i^2 \right]\right| \leq \frac{\epsilon}{2}. \label{eq:conNetSize}
\end{align}

Moreover, since each $[\vc{U}]_{i,j}$ is independent and identically distributed from a symmetric stable distribution with tail-index $\alpha_l$, by \cite[Theorem 2.7]{teimouri2020asymptotic}, as $h_{l} \to \infty$, for each $k=1,\dots,h_{l-1}$, the eigenvalue $\lambda'_{l,k}$ weakly converges to a  random variable $\xi_{l,k}$, where the collection $\{\xi_{l,k}\}_{k=1}^{h_{l-1}}$ is independent and identically distributed from a positive stable distribution with tail-index $\alpha_l/2$. Denote by $\boldsymbol{\xi}_l := [\xi_{l,1}, \dots, \xi_{l,h_{l-1}}] \in \rset^{h_{l-1}}$ the random vector containing the limiting i.i.d.\ random variables. 

We will now construct a sequence of $\{h_{l,0}\}_{l=1}^{L-1}$ such that the claims will follow for $h_l \geq h_{l,0}$. Let us start from the second layer, i.e., set $l=2$.
Then, by Lemma~\ref{lem:compress}, for any $\veps_2 >0$, $\kappa_2\in(0,1)$, and $p\geq  \alpha_2/2$, there exists $h'_{1,0} \in \mathbb{N}_+$, such that $h_{1} \geq h'_{1,0}$ implies:
\begin{align}
    \mathbb{P} \left[ \frac{\|\boldsymbol{\xi}_{2}^{(\kappa_{2} h_{1})} - \boldsymbol{\xi}_{2}  \|_p}{\|\boldsymbol{\xi}_{2}\|_p} \leq \veps_{2}^2 \right] =  1.
\end{align}
Let $h_{1,0} = h'_{1,0} \vee \hat{h}_{1,0}$. Having fixed $h_1 \geq h_{1,0}$, we now iterate the following argument from $l=2$ to $l=L-1$ sequentially. Due to the weak convergence of the eigenvalues, we have:
\begin{align}
    \lim_{h_{l} \to \infty} \mathbb{P} \left[ \frac{\|\boldsymbol{\lambda}_l^{'(\kappa_l h_{l-1})} - \boldsymbol{\lambda}'_l  \|_p}{\|\boldsymbol{\lambda}'_l\|_p} \leq \veps_l^2 \right] &= \mathbb{P} \left[ \frac{\|\boldsymbol{\xi}_l^{(\kappa_l h_{l-1})} - \boldsymbol{\xi}_l  \|_p}{\|\boldsymbol{\xi}_l\|_p} \leq \veps_l^2 \right]\\
    &= 1.
\end{align}
Hence, combining with \eqref{eq:conNetSize}, for any $\epsilon>0$, there exists $h''_{l,0} \in \mathbb{N}_+$, such that $h_{l} \geq h''_{l,0} \vee \hat{h}_{l,0} $ implies 
\begin{align}
    \label{eqn:specprun_interm}
    \mathbb{P} \left[ \frac{\|\boldsymbol{\lambda}_l^{(\kappa_l h_{l-1})} - \boldsymbol{\lambda}_l  \|_p}{\|\boldsymbol{\lambda}_l\|_p} \leq \veps_l^2 \right] \geq 1-\epsilon.
\end{align}
If $l=L-1$, set $h_{l,0} = h''_{l,0} \vee \hat{h}_{l} \vee  h_{l-1}$. If $l \leq L-2$, repeat the previous argument to find a $h'_{l,0}$, such that $h_{l} \geq h'_{l,0}$ implies
\begin{align}
    \mathbb{P} \left[ \frac{\|\boldsymbol{\xi}_{l+1}^{(\kappa_{l+1} h_{l})} - \boldsymbol{\xi}_{l+1}  \|_p}{\|\boldsymbol{\xi}_{l+1}\|_p} \leq \veps_{l+1}^2 \right] =  1,
\end{align}
and set $h_{l,0} = h'_{l,0} \vee h''_{l,0} \vee \hat{h}_{l,0} \vee h_{l-1} $. This proves the first claim.

To prove the second claim, first notice that we can set $p=1$ as $\max_l\alpha_l<2$, hence $p=1\geq \alpha_l/2$ for all $l$. Now, fix $l\in \llbracket 2, L-1 \rrbracket$, and for a given $h_{l-1}$ and $h_l$, consider the singular value decomposition of $\vc{W}_l$ as:
\begin{align}
    \vc{W}_l = \vc{U} \boldsymbol{\Sigma} \vc{V}^\top,
\end{align}
and define $\vc{W}_l^{[\kappa_l h_l-1]} := \vc{U} \boldsymbol{\Sigma}^{(\kappa_l h_{l-1})} \vc{V}^\top$, where $\boldsymbol{\Sigma}^{(\kappa_l h_{l-1})}$ is the diagonal matrix whose diagonal entries contain the $\lceil \kappa_l h_{l-1} \rceil$ largest singular values (i.e., prune the diagonal part of $\boldsymbol{\Sigma}$). By using \eqref{eqn:specprun_interm} and the fact that the Schatten $2$-norm coincides with the Frobenius norm, we have:
\begin{align}
\|\vc{W}_l\|^2 = h_l^{1/\alpha_l} \|\boldsymbol{\lambda}_l\|_1, \quad \text{and} \quad \|\vc{W}_l^{[\kappa_l h_l-1]} - \vc{W}_l\|^2 = h_l^{1/\alpha_l} \|\boldsymbol{\lambda}_l^{(\kappa_l h_{l-1})} - \boldsymbol{\lambda}_l\|_1.
\end{align}
Hence, we conclude that for $h_l \geq h_{l,0}$, the following inequality holds for $l\in \llbracket 2, L-1 \rrbracket$:
\begin{align}
    \mathbb{P} \left[ \frac{\|\vc{W}^{[\kappa_l h_l-1]}_l - \vc{W}_l\|}{\|\vc{W}_l\|} \leq \veps_l \right] \geq 1-\epsilon.
\end{align}
This concludes the proof.
\end{proof}

\subsection{Proof of Theorem~\ref{th:normPruning}}
\begin{proof}
For $l \in \llbracket 2, L\rrbracket$, let $v_{l,i}=\|\vc{W}_l(i)\|_p$, where $\vc{W}_l(i) \in \mathbb{R}^{h_l}$ is the $i$-th column of $\vc{W}_l \in \mathbb{R}^{h_l \times h_{l-1}}$ for $i \in \llbracket 1, h_{l-1} \rrbracket$. Note that by definition, for any    $\kappa_l \in (0,1)$ 
\begin{align*}
    \|\vc{v}_l\|_p&=\|\vc{w}_l\|_p,~ \text{ and }~
    \|\vc{v}_l^{(\kappa_l h_{l-1})}-\vc{v}_l\|_p = \|\vc{w}_l^{\{\kappa_l h_{l-1}\},p}-\vc{w}_l\|_p.
\end{align*}
Hence, it suffices to show that for any  $\epsilon>0$ and $\veps_l >0$, there exists $h_{l-1,0}$ such that for $h_{l-1} \geq h_{l-1,0}$,
\begin{align}
    \mathbb{P}\left[\left(\|\vc{v}_l^{(\kappa_l h_{l-1})}-\vc{v}_l\|_p\right)/\|\vc{v}_l\|_p\leq \varepsilon_l\right] \geq 1- \epsilon.  \label{eq:normPruning0}
\end{align}

As network dimensions grow, due to HML property, $w_{l,j}$, $j \in \llbracket 1,d_l \rrbracket$ converges in distribution to an i.i.d. heavy-tailed random variable with tail-index $\alpha_l$. Hence, as $\{h_{l-1}\}_{l=2}^{L}$ grows, $v_{l,i}$ also converges in distribution to a heavy-tailed random variable, denoted as $\vc{x}_l=(x_{l,1},\ldots,x_{l,h_{l-1}})\in \mathbb{R}^{h_{l-1}}$ ,  with i.i.d. elements and tail-index $\alpha_l \in (0,2)$. Thus, there exists $\{h'_{l-1,0}\}_{l=2}^L\colon h'_{l-1,0} \in \mathbb{N}$ such that for $h_{l-1} \geq h'_{l-1,0}$, $l \in \llbracket 2,L \rrbracket$,
\begin{align}
    \mathbb{P}\left[\left(\|\vc{x}_l^{(\kappa_l h_{l-1})}-\vc{x}_l\|_p/\|\vc{x}_l\|_p\right)\leq \varepsilon_l\right] - \mathbb{P}\left[\left(\|\vc{v}_l^{(\kappa_l h_{l-1})}-\vc{v}_l\|_p/\|\vc{v}_l\|_p\right)\leq \varepsilon_l\right] \leq \epsilon. \label{eq:normproof2}
\end{align}

Moreover due to Lemma~\ref{lem:compress}, there exists $\{h''_{l-1,0}\}_{l=2}^L\colon h''_{l-1,0} \in \mathbb{N}$ such that for $h_{l-1} \geq h''_{l-1,0}$, 
\begin{align}
   \mathbb{P}\left[\left(\|\vc{x}_l^{(\kappa_l h_{l-1})}-\vc{x}_l\|_p/\|\vc{x}_l\|_p\right)\leq \varepsilon_l\right]=1 \label{eq:normproof3}
\end{align}
The results follows from \eqref{eq:magproof4} and \eqref{eq:magproof3} and by choosing $h_{l-1,0} \geq \max(h'_{l-1,0},h''_{l-1,0})$.
\end{proof}

\subsection{Proof of Theorem~\ref{th:generalizationPrunedNetwork}} \label{sec:proofGenPrun}

Let
\begin{align*}
D\left(y,f_{\vc{w}}(\vc{x})\right)\coloneqq f_{\vc{w}}(\vc{x})[y]-\max \limits_{j \neq y} f_{\vc{w}}(\vc{x})[j].    
\end{align*}
Define the surrogate loss function $ \ell_{\gamma,\tau}\colon \mathcal{Y}\times \mathcal{Y} \mapsto [0,1]$, with margin loss $\gamma\geq 0$ and continuity margin $\tau > 0$, for the multiclass classifier $f_{\vc{w}}$ as:
\begin{align}
    \ell_{\gamma,\tau}\left(y,f_{\vc{w}}(\vc{x})\right)\coloneqq\begin{cases}1,& \text{ if } D\left(y,f_{\vc{w}}(\vc{x})\right) \leq \gamma,\\
    1-\frac{D\left(y,f_{\vc{w}}(\vc{x})\right)-\gamma}{\tau},& \text{ if } \gamma < D\left(y,f_{\vc{w}}(\vc{x})\right) \leq \gamma+\tau,\\
    0,& \text{ if } \gamma+\tau < D\left(y,f_{\vc{w}}(\vc{x})\right).
    \end{cases} \label{def:lossContMargin}
\end{align}
Note that $\ell_{\gamma}\left(y,f_{\vc{w}}(\vc{x})\right)= \ell_{\gamma,0}\left(y,f_{\vc{w}}(\vc{x})\right)$. Population and empirical risks of a hypothesis $\vc{w}$ are denoted by $\Rcal_{\gamma,\tau}({\vc{w}})$ and $\hat{\Rcal}_{\gamma,\tau}({\vc{w}})$, respectively.

\begin{proof}
Recall that for all $l \in \llbracket 1,L \rrbracket$,  $\vc{\widehat{w}}_l:=\vc{w}_l^{\left(\kappa_l d_l\right)}$. Denote by $A$ the event that $\vc{w}$ is compressible, \emph{i.e.}   when for all $ l \in \llbracket 1,L \rrbracket$,   $\|\vc{\widehat{w}}_l - \vc{w}_l\| \leq \veps \|\vc{w}_l\|$. Denote its complement by $A^C$ and note that $\mathbb{P}(A^C)\leq \epsilon$, where the probability is with respect to $P_{S,\vc{w}}$.

Fix $\delta,\tau>0 $. We show that with probability of at least $1-2e^{-\kappa d/2}-\delta-\epsilon$:
\begin{align}
        \Rcal_{0,\tau}({\vc{\widehat{w}}}) & \leq \hat{\Rcal}_{0,\tau}({\vc{\widehat{w}}})+ \left(12 \mathcal{L}(\tau,\delta)R(\delta)+\sqrt{d}\right) \sqrt{\frac{ (\kappa+\epsilon_{\kappa}) d \log(n)}{n}}, \label{eq:genPruneq1}
\end{align}
and moreover $\|\vc{w}\|\leq R(\delta)$ and $\|\vc{\widehat{w}}_l - \vc{w}_l\| \leq \veps \|\vc{w}_l\|$ for all $ l \in \llbracket 1,L \rrbracket$, simultaneously. Then, under the latter two conditions, 
\begin{align*}
      \Rcal_{0}({\vc{\widehat{w}}}) & \leq   \Rcal_{0,\tau}({\vc{\widehat{w}}}), \\
      \hat{\Rcal}_{0,\tau}({\vc{\widehat{w}}}) & \leq \hat{\Rcal}_{ \gamma(\delta,\tau)}({\vc{w}}),
\end{align*}
using Lemma~\ref{lem:prunOutputDiff},  Definition~\ref{def:lossContMargin}, and  Lemma~\ref{lem:diffBound} that bounds the Lipschitz coefficient of the network. This completes the proof.

\begin{lem} \label{lem:diffBound}
Suppose that for $l \in \llbracket 1,L\rrbracket$ and a given $\upsilon_l>0$, we have  $\|\vc{w}_l-\vc{w}'_l\|\leq \upsilon_l$.
\begin{itemize}
    \item[i.] Then, for any $(\vc{x},y)$, the following relations hold:
    \begin{align*}
        \| f_{\vc{w}}(\vc{x})-f_{\vc{w}'}(\vc{x})\| &\leq B \prod \limits_{l=1}^L \left(\|\vc{w}_l\|+\upsilon_l \right)-B\prod \limits_{l=1}^L \|\vc{w}_l\|, \\
        \big| \ell_{0,\tau}(y, f_{\vc{w}}(\vc{x}))-\ell_{0,\tau}(y, f_{\vc{w}'}(\vc{x})) \big| &\leq \frac{\sqrt{2}}{\tau} \| f_{\vc{w}}(\vc{x})-f_{\vc{w}'}(\vc{x})\|.
    \end{align*}
    \item[ii.]
    In particular, if  $\upsilon_l=\upsilon \leq R/\sqrt{L}$ for $l \in \llbracket 1,L \rrbracket$ and if $\|\vc{w}\|\leq R$, then 
    \begin{align*}
        \| f_{\vc{w}}(\vc{x})-f_{\vc{w}'}(\vc{x})\| &\leq  BL\left(\frac{2R}{\sqrt{L}}\right)^{L-1}\upsilon,\\
        \big| \ell_{0,\tau}(y, f_{\vc{w}}(\vc{x}))-\ell_{0,\tau}(y, f_{\vc{w}'}(\vc{x})) \big|  &\leq \frac{BL\sqrt{2}}{\tau}\left(\frac{2R}{\sqrt{L}}\right)^{L-1}\upsilon \eqqcolon \mathcal{L}(\tau,\delta) \upsilon .
    \end{align*}
\end{itemize}
\end{lem}
Hence, it remains to show \eqref{eq:genPruneq1} together with the conditions $\|\vc{w}\|\leq R(\delta)$ and $\|\vc{\widehat{w}}_l - \vc{w}_l\| \leq \veps \|\vc{w}_l\|$, for $ l \in \llbracket 1,L \rrbracket$, hold with probability at least $1-2e^{-\kappa d}-\delta -\epsilon$. Now, first whenever $\|\vc{\widehat{w}}\|\leq R(\delta)$, we discretize $\vc{\hat{w}}$. Let
\begin{align*}
    \widehat{\mathcal{W}}(R(\delta),d,\kappa)\coloneqq \left\{\vc{\widehat{w}} \in \mathbb{R}^d \big| \|\vc{\widehat{w}}\|\leq R(\delta), \|\vc{\widehat{w}}\|_0 \leq \kappa d \right\},
\end{align*} 
where  $\|\vc{\widehat{w}}\|_0$ denotes the number of non-zero components of $\vc{\widehat{w}}$. Assume that $\widetilde{\mathcal{W}}(R(\delta),d,\kappa)$ is a discretization of this space with $\upsilon>0$ precision, \emph{i.e.} for every $\vc{\widehat{w}} \in   \widehat{\mathcal{W}}(R(\delta),d,\kappa)$, there exists $ \widetilde{\vc{w}} \in   \widetilde{\mathcal{W}}(R(\delta),d,\kappa)$ satisfying $\| \widetilde{\vc{w}}-\vc{\widehat{w}} \|\leq \upsilon$. Among all such coverings, consider the one with  minimum number of $\mathcal{N}_{\upsilon}$ points.

\begin{lem} \label{lem:disPointsComp}
$\mathcal{N}_{\upsilon} \leq e^{d h_b(\kappa)} \left(\frac{3 R(\delta) }{\upsilon}\right)^{\kappa d}$.
\end{lem}

Note that in general $\|\vc{\widehat{w}}\| \leq\|\vc{w}\|$. Let $ l \coloneqq c_1(\delta,\tau)\sqrt{\frac{(\kappa+\epsilon_{\kappa}) d \log(n)}{n}}$ and $\upsilon =\frac{a l}{4\mathcal{L}(\tau,\delta)}$, where $a\coloneqq \frac{12 \mathcal{L}(\tau,\delta)R(\delta)}{12 \mathcal{L}(\tau,\delta)R(\delta)+\sqrt{d}}$. Then
\begin{align}
    \mathbb{P}\bigg( \big| \Rcal_{0,\tau}({\vc{\widehat{w}}})&- \hat{\Rcal}_{0,\tau}({\vc{\widehat{w}}})\big| \geq l \, \bigcup \, \|\vc{w}\| > R(\delta) \,\bigcup\, A^C \bigg) \nonumber \\
    & \leq \mathbb{P}\left( \big| \Rcal_{0,\tau}({\vc{\widehat{w}}})- \hat{\Rcal}_{0,\tau}({\vc{\widehat{w}}})\big| \geq l \,\bigcap \, \|\vc{w}\|\leq R(\delta)\right)+
    \mathbb{P}\left( \|\vc{w}\| \geq R(\delta)\right) + \mathbb{P}\left( A^C \right) \nonumber \\
    & \leq \mathbb{P}\left( \big| \Rcal_{0,\tau}({\vc{\widehat{w}}})- \hat{\Rcal}_{0,\tau}({\vc{\widehat{w}}})\big| \geq l \,\bigcap \, \|\vc{w}\|\leq R(\delta)\right)+\delta+\epsilon \nonumber \\
    & \leq \mathbb{P}\left(\sup \limits_{\vc{\widehat{w}} \in \widehat{\mathcal{W}}(R(\delta),d,\kappa) } \big| \Rcal_{0,\tau}({\vc{\widehat{w}}})- \hat{\Rcal}_{0,\tau}({\vc{\widehat{w}}})\big| \geq l/2 \right)+\delta+\epsilon \nonumber \\
    & \stackrel{(a)}{\leq} \mathbb{P}\left(\max \limits_{\vc{\widetilde{w}} \in  \widetilde{\mathcal{W}}(R(\delta),d,\kappa)} \big| \Rcal_{0,\tau}(\vc{\widetilde{w}})- \hat{\Rcal}_{0,\tau}(\vc{\widetilde{w}})\big| \geq l(1-a)/2\right)+\delta+\epsilon\nonumber\\
    &\stackrel{(b)}{\leq}  \mathcal{N}_{\upsilon} \max \limits_{\widetilde{\vc{w}}\in  \widetilde{\mathcal{W}}(R(\delta),d,\kappa)} \mathbb{P}\left( \big| \Rcal_{0,\tau}(\vc{\widetilde{w}})- \hat{\Rcal}_{0,\tau}(\vc{\widetilde{w}})\big| \geq l(1-a)/2 \right)+\delta+\epsilon\nonumber \\
    &\stackrel{(c)}{\leq}2  \mathcal{N}_{\upsilon}  \exp\left(-n l^2(1-a)^2/2\right) +\delta+\epsilon \nonumber \\
        &\stackrel{(d)}{\leq} 2   \exp\left(-n l^2(1-a)^2/2+\kappa d\log\left(\frac{12  \mathcal{L}(\tau,\delta) R(\delta) }{al}\right)+dh_b(\kappa)\right) +\delta+\epsilon, 
    \label{eq:genBound1}
\end{align}
where $(a)$ is derived since
\begin{align*}
    \big|\Rcal_{0,\tau}({\vc{\widehat{w}}})- \hat{\Rcal}_{0,\tau}({\vc{\widehat{w}}})\big|& \leq  \big|\Rcal_{0,\tau}({\vc{\widetilde{w}}})- \hat{\Rcal}_{0,\tau}({\vc{\widetilde{w}}})\big|+2 \mathcal{L}(\tau,\delta) \upsilon=\big|\Rcal_{0,\tau}({\vc{\widetilde{w}}})- \hat{\Rcal}_{0,\tau}({\vc{\widetilde{w}}})\big|+a l/2.
\end{align*}
by Lemma~\ref{lem:diffBound}, given that $\upsilon \leq R/\sqrt{L}$, and the triangle inequality. The inequality $(b)$ is obtained by applying the union bound, $(c)$ is derived using Hoeffding's inequality and since loss is bounded by 1, and $(d)$ is due to Lemma~\ref{lem:disPointsComp}. \\
It remains to show that the term in the exponent in \eqref{eq:genBound1} is upper bounded by $-\kappa d/2$ and $\upsilon \leq R/\sqrt{L}$.
\begin{align}
    -n l^2(1-a)^2/2+\kappa d&\log\left(\frac{12  \mathcal{L}(\tau,\delta) R(\delta) }{al}\right)+dh_b(\kappa)\\
    =&-\frac{ c_1(\delta,\tau)^2 (1-a)^2 \kappa d \log(n) }{2}+\frac{\kappa d}{2}\log(n) \label{eq:genConclusionComp1}\\
    &+\kappa d\log\left(\frac{12  \mathcal{L}(\tau,\delta) R(\delta) }{a c_1(\delta,\tau)\sqrt{d}}\right)
    \label{eq:genConclusionComp2}\\
    &-\frac{ c_1(\delta,\tau)^2 (1-a)^2 \epsilon_{\kappa} d \log(n) }{2}-\frac{\kappa d}{2} \log\left(\kappa+\epsilon_{\kappa}\right) +dh_b(\kappa) \label{eq:genConclusionComp3}\\
    &-\frac{\kappa d}{2}\log \log(n). \label{eq:genConclusionComp4}
\end{align}
It can be verified that \eqref{eq:genConclusionComp1} and  \eqref{eq:genConclusionComp2} are non-positive when $c_1(\delta,\tau) \geq \left(12 \mathcal{L}(\tau,\delta)R(\delta)+\sqrt{d}\right)/\sqrt{d}$. Moreover, with this choice of $c_1(\delta,\tau)$, \eqref{eq:genConclusionComp3} is non-positive for $ \epsilon_{\kappa}  = \left(2 h_b(\kappa)-\kappa\log \left(\kappa\right)\right)/\left(\log(n)\right)$. Finally, \eqref{eq:genConclusionComp3} is less than $-\kappa d/2$, for $n\geq 16$. 

Finally with the chosen value of $\upsilon$,  $\upsilon \leq R/\sqrt{L}$ holds if $n/\log(n) \geq 10L$. This completes the proof.
\end{proof}

\subsection{Proof of Corollary~\ref{cor:genCompressStable}}

For notation convenience, let $\sas_n(\sigma)\equiv \sas(\sigma_{\alpha} \sigma)$,  where $n$ stands for normalized and $\sigma_{\alpha}\coloneqq \left(2\Gamma(-\alpha)\cos((2-\alpha)\pi/2)\right)^{1/\alpha}$. First, we state the Corollary for a more  general case, and then we state the proof of this general result.

\begin{cor} \label{cor:genCompressStableGeneral} Assume that for $l \in \llbracket 1,L \rrbracket$ and $i \in \llbracket 1,d_l \rrbracket$, the conditional distribution of ${w}_{l,i} \stackrel{\mathclap{\mbox{\tiny{i.i.d.}}}}{\sim} \mathcal{S}\alpha_l \mathcal{S}_n(\sigma_l) $  with $\alpha_l \in (1,2)$. Further assume that $\sigma^2 \coloneqq \sum\nolimits_{l=1}^L (d_l/d) \sigma_l^2$ and $\{\alpha_l\}_{l=1}^L$  do not depend on $S$. Then for every $\veps >0$, $\kappa_l \in (0,1)$, $l \in \llbracket 1,L \rrbracket$, and $\beta>0$,  there exists $d_{l,0} \in \mathbb{N}$, such that for every $n \colon n/\log(n) \geq 10L$ and  $\tau>0$, with probability  at least $1-3d^{-\beta}$,
\begin{align}
         \Rcal_{0}({\vc{\widehat{w}}}) & \leq \hat{\Rcal}_{\gamma}({\vc{w}})+ \left(a(\alpha)\sigma^{L} d^{\frac{L\left(\alpha+2\beta+2\right)}{2\alpha}} /\tau+\sqrt{d}\right)\sqrt{\left(\kappa+\epsilon_{\kappa}\right)  \log(n)/n},  \label{eq:generalizationPrunedNetworkStableGen}
\end{align}
where  $\{\vc{\widehat{w}_l}\}_l=\{\vc{w}_l^{(\kappa d_l)}\}_l$,  $\alpha{\coloneqq}\min_{l}\alpha_l$, 
\begin{align*}
a(\alpha) \coloneqq  6\sqrt{2} B  6^{L} 4^{L/\alpha} /  L^{(L-3)/2},~~~\gamma \coloneqq \tau+b_{\varepsilon}(\alpha) \sigma^{L}  d^{\frac{L\left(\alpha+2\beta+2\right)}{2\alpha}}/\tau    
\end{align*}
and  $b_{\varepsilon}(\alpha) \coloneqq \sqrt{2} B 3^{L} 4^{L/\alpha} \left((1+\varepsilon)^{L}-1\right) /  L^{L/2}$.
\end{cor}

\begin{proof} 
First, given any $S$, we bound the term $R_S(\delta)$, defined as
\begin{align}
        R_S(\delta) \coloneqq \inf \left\{ R \colon \mathbb{P}\left( \|\vc{w}\| \geq R \big| S\right) \leq \delta \right\}.
\end{align}
\begin{lem}   \label{lem:normBoundStable}
If for $l \in \llbracket 1,L \rrbracket$, $\vc{x}_l$ is an i.i.d.\ $d_l$-dimensional vector with with $\mathcal{S}\alpha_l \mathcal{S}_n(\sigma_l)$ distributions and $\alpha_l \in (1,2)$, then  for $\delta < 2d(2-\max_{l}\alpha_l)^{\alpha}$
  \begin{align*}
       \inf \left\{ R \colon \mathbb{P}\left( \|\vc{x}\| \geq R \right) \leq \delta \right\}  \leq 3\sigma \sqrt{d}  \left(\frac{4 d }{\delta}\right)^{1/\alpha},
        \end{align*}
 where $\sigma \coloneqq \sqrt{\sum_{l=1}^L (d_l/d) \sigma_l^2}$ and $\alpha \coloneqq \min_{l}\alpha_l$.
\end{lem}
Hence, $ R_S(\delta)   \leq 3\sigma \sqrt{d}  \left(\frac{4 d }{\delta}\right)^{1/\alpha
}$. Since $\sigma$, $\alpha$, and $\max_{l}\alpha_l$ do not depend on $S$, then this bound is the same for all $S$. Thus, for $\delta < 2d(2-\max_{l}\alpha_l)^{\alpha}$,
\begin{align}
        R(\delta)   \leq 3\sigma \sqrt{d}  \left(\frac{4 d }{\delta}\right)^{1/\alpha}. \label{eq:boundRdelta}
\end{align}
Next, due to Lemma~\ref{lem:compress}, the assumption \cref{asmp:magp} holds  for any $\veps >0$ and $\{\kappa_l\}_{l=1}^L \colon \kappa_l \in (0,1)$ and some $\{d_{l,0}\}_{l=1}^L\colon d_{l,0} \in \mathbb{N}$,  with $\epsilon=0$. Note that $d_{l,0}$  does not depend on $S$ as $\{\alpha_l\}_{l=1}^L$ is independent of $S$. The proof follows now by  Theorem~\ref{th:generalizationPrunedNetwork} with $\delta=d^{-\beta}$ and using the relation~\eqref{eq:boundRdelta} when  $ 2d_0^{\beta+1}(2-\max_l \alpha_l)^{\alpha}\geq 1$ and $ d_0/\log(d_0)\geq \beta/\kappa$.
\end{proof}

\subsection{Proof of Theorem~\ref{th:generalizationOriginal}} \label{sec:generalizationOriginal}

\begin{proof}
The proof of this theorem is similar to the proof of  Theorem~\ref{th:generalizationPrunedNetwork}.  Fix $\delta,\tau>0 $. We show that with probability of at least $1-2e^{-d/2}-\delta-\epsilon$:
\begin{align}
        \Rcal_{0,\tau}({\vc{w}}) & \leq \hat{\Rcal}_{0,\tau}({\vc{w}})+ \max\left(2,24 \rho_{\varepsilon}(\kappa,d) \mathcal{L}(\tau,\delta)R(\delta)/\sqrt{d}\right)\sqrt{d\log(n)/n},
\end{align}
where $\Rcal_{0,\tau}(\cdot)$ and $\hat{\Rcal}_{0,\tau}(\cdot)$ are defined in \eqref{def:lossContMargin}.
The claim follows then by noting that
\begin{align*}
      \Rcal_{0}({\vc{w}}) & \leq   \Rcal_{0,\tau}({\vc{w}}), \\
      \hat{\Rcal}_{0,\tau}({\vc{w}}) & \leq \hat{\Rcal}_{\tau}({\vc{w}}),
\end{align*}
due to \eqref{def:lossContMargin}.

Recall that for all $l \in \llbracket 1,L \rrbracket$,  $\vc{\widehat{w}}_l:=\vc{w}_l^{\left(\kappa_l d_l\right)}$. Denote by $A$ the event that $\vc{w}$ is compressible, \emph{i.e.}   when for all $ l \in \llbracket 1,L \rrbracket$,   $\|\vc{\widehat{w}}_l - \vc{w}_l\| \leq \veps \|\vc{w}_l\|$. Denote its complement by $A^C$ and note that $\mathbb{P}(A^C)\leq \epsilon$, where the probability is with respect to $P_{\vc{w}}$.

In the following, first we discretize $\vc{w}$ whenever $\|\vc{w}\|\leq R(\delta)$.  Let 
\begin{align*}
    \mathcal{W}(R(\delta),\varepsilon,d,\kappa)\coloneqq \left\{\vc{w} \in \mathbb{R}^d \big| \|\vc{w}\|\leq R(\delta), \|\vc{w}^{(\kappa d)}-\vc{w}\| \leq \varepsilon R(\delta) \right\}.
\end{align*} 

Assume that $\mathcal{W}'(R(\delta),\varepsilon,d,\kappa)$ is the discretization of this space  with $\upsilon>0$ precision, \emph{i.e.}  for every $\vc{w} \in   \mathcal{W}(R(\delta),\varepsilon,d,\kappa)$, there exists $ \vc{w}' \in   \mathcal{W}'(R(\delta),\varepsilon,d,\kappa)$ satisfying $\| \vc{w}'-\vc{w} \|\leq \upsilon$. Among all such coverings, consider the one with  minimum number of $\mathcal{N}'_{\upsilon}$ points.

\begin{lem} \label{lem:disPointsOriginal}  For $d\geq 10$, if $\upsilon< \varepsilon R(\delta)$, then $\mathcal{N}'_{\upsilon} \leq  \left(\frac{3 \rho_{\varepsilon}(\kappa,d) R }{\upsilon}\right)^{d}$.
\end{lem}
Similar to the proof of Theorem~\ref{th:generalizationPrunedNetwork} and by letting $ l \coloneqq c_2(\delta,\tau,\kappa)\sqrt{\frac{ d \log(n)}{n}}$ and  $\upsilon =\frac{a l}{4\mathcal{L}(\tau,\delta)}$, where $a\coloneqq \frac{12 \rho_{\varepsilon}(\kappa,d) \mathcal{L}(\tau,\delta)R(\delta)}{12 \rho_{\varepsilon}(\kappa,d) \mathcal{L}(\tau,\delta)R(\delta)+\sqrt{d}}$, 
\begin{align}
 \mathbb{P}\bigg( \big| \Rcal_{0,\tau}({\vc{w}})-& \hat{\Rcal}_{0,\tau}({\vc{w}})\big| \geq l  \bigg)\nonumber
 \\&\leq     \mathbb{P}\bigg( \big| \Rcal_{0,\tau}({\vc{w}})- \hat{\Rcal}_{0,\tau}({\vc{w}})\big| \geq l \,  \bigcup \, \|\vc{w}\| > R(\delta) \,  \bigcup \, A^C\bigg)\nonumber \\
  & \leq \mathbb{P}\left( \big| \Rcal_{0,\tau}({\vc{w}})- \hat{\Rcal}_{0,\tau}({\vc{w}})\big| \geq l \, \bigcap \,\|\vc{w}\|\leq R(\delta)\, \bigcap \, A\right)+\delta+\epsilon \nonumber\\
   & \leq \mathbb{P}\left(\sup \limits_{\vc{w} \in {\mathcal{W}}(R(\delta),\varepsilon,d,\kappa) } \big| \Rcal_{0,\tau}({\vc{w}})- \hat{\Rcal}_{0,\tau}({\vc{w}})\big| \geq l/2 \right)+\delta+\epsilon \nonumber \\
    &\stackrel{(a)}{\leq}  \mathcal{N}'_{\upsilon} \max \limits_{\vc{w}'\in   \mathcal{W}'(R(\delta),\varepsilon,d,\kappa)} \mathbb{P}\left( \big| \Rcal_{0,\tau}(\vc{w}')- \hat{\Rcal}_{0,\tau}(\vc{w}')\big| \geq l(1-a)/2 \right)+\delta+\epsilon\nonumber \\
    &\stackrel{(b)}{\leq} 2   \exp\left(-\frac{n l^2(1-a)^2}{2}+d\log\left(\frac{12 \rho_{\varepsilon}(\kappa,d)  \mathcal{L}(\tau,\delta) R(\delta) }{a l}\right)\right) +\delta+\epsilon,
    \label{eq:genBound2}
\end{align}
where $(a)$ holds when $\upsilon \leq R/\sqrt{L}$ and $(b)$ holds using Lemma~\ref{lem:disPointsOriginal} if $\upsilon<\varepsilon R(\delta)$.

It remains to show that the term in the exponent in \eqref{eq:genBound2} is upper bounded by $-d/2$, $\upsilon<\varepsilon R(\delta)$, and $\upsilon \leq R/\sqrt{L}$. To show the first claim, we can write
\begin{align}
   -\frac{n l^2(1-a)^2}{2}+d\log\left(\frac{12 \rho_{\varepsilon}(\kappa,d)  \mathcal{L}(\tau,\delta) R(\delta) }{a l}\right)=&-\frac{ c_2(\delta,\tau,\kappa)^2 (1-a)^2 d \log(n) }{2}+\frac{d}{2}\log(n) \label{eq:genConclusion1}\\
    &+d\log\left(\frac{12 \rho_{\varepsilon}(\kappa,d)  \mathcal{L}(\tau,\delta) R(\delta) }{a c_2(\delta,\tau,\kappa)\sqrt{d}}\right) \label{eq:genConclusion2}\\
    &-\frac{d}{2}\log \log(n). \label{eq:genConclusion3}
\end{align}
It can be verified that \eqref{eq:genConclusion1} and \eqref{eq:genConclusion2} are non-positive when 
\begin{align*}
c_2(\delta,\tau,\kappa) = \left(12 \rho_{\varepsilon}(\kappa,d) \mathcal{L}(\tau,\delta)R(\delta)+\sqrt{d}\right)/\sqrt{d},    
\end{align*}
and \eqref{eq:genConclusion3} is less than $-d/2$, for $n\geq 16$. 

To verify $\upsilon<\varepsilon R(\delta)$, where  $\upsilon=\frac{a l}{4\mathcal{L}(\tau,\delta)}$, we have
\begin{align*}
    \upsilon&=
     \frac{12 \rho_{\varepsilon}(\kappa,d) \mathcal{L}(\tau,\delta)R(\delta)}{12 \rho_{\varepsilon}(\kappa,d) \mathcal{L}(\tau,\delta)R(\delta)+\sqrt{d}} \times \frac{ 12 \rho_{\varepsilon}(\kappa,d) \mathcal{L}(\tau,\delta)R(\delta)+\sqrt{d}}{\sqrt{d}} \times \frac{\sqrt{d\log(n)/n}}{4\mathcal{L}(\tau,\delta)}\\
    &=3 \rho_{\varepsilon}(\kappa,d)  R(\delta) \sqrt{\log(n)/n}\\
    &\leq 9 \varepsilon^{1-\kappa} R(\delta) \sqrt{\log(n)/n} \stackrel{(a)}{\leq} \varepsilon R(\delta),
\end{align*}
where $(a)$ holds when $\varepsilon^{\kappa} \geq 9\sqrt{\log(n)/n}$. 

Moreover, with the chosen value of $\upsilon$,  $\upsilon \leq R/\sqrt{L}$ holds if $n/\log(n) \geq 9L$. Finally note that $ \left(12 \rho_{\varepsilon}(\kappa,d) \mathcal{L}(\tau,\delta)R(\delta)+\sqrt{d}\right)\leq \max(24 \rho_{\varepsilon}(\kappa,d) \mathcal{L}(\tau,\delta)R(\delta),2\sqrt{d})$. This completes the proof.
\end{proof}

\subsection{Proof of Proposition~\ref{prop:compressionRate}}
\begin{proof}
 For ease of notations, for $\kappa,\varepsilon>0$ and $\vc{w} 
 \in \mathbb{R}^d$, let
\begin{align}
    \varepsilon_p(\vc{w},\kappa) & \coloneqq \|\mathbf{w}^{(\kappa d)} - \mathbf{w}\|_p / \|\mathbf{w}\|_p ~\text{ and hence  }~\kappa_p(\vc{w},\varepsilon)\coloneqq \min \left\{\kappa \colon \varepsilon_p(\vc{w},\kappa) \leq \varepsilon \right\}.
\end{align}
Let $\vc{x}=(|w_1|,\ldots,|w_d|)$ and let $x_{d,i}$ be the corresponding ordered sequence, \textit{i.e.}
\begin{align*}
    x_{d,1}\geq x_{d,2} \geq \cdots \geq x_{d,d}.
\end{align*}
Let
\begin{align*}
    \vc{y}^d=\frac{1}{a_d}\left(x_{d,1},x_{d,2},\ldots,x_{d,d},0,0,\ldots \right) \in \mathbb{R}^{\infty},
\end{align*}
where $a_d$ is a normalizing constant defined in \cite[Equation~3]{lepage1981}. Moreover let $e_i$, $i=1,2,\ldots$, be i.i.d. standard exponential random variables with partial sum $\Gamma_i\coloneqq \sum_{l=1}^i e_l$ and let $z_i(\alpha)\coloneqq \Gamma_i^{-1/\alpha}$. Then, due to \cite[Lemma~1]{lepage1981},
\begin{align*}
    \lim \limits_{d\rightarrow \infty} \vc{y}^d \stackrel{\mathbbm{d}}{=} (z_1(\alpha),z_2(\alpha),\ldots).
\end{align*}
where $\mathbbm{d}$ denotes convergence in distribution.

First, we show that for any $\kappa>0$,  $\varepsilon_p\left(\vc{z}^d(\alpha),\kappa\right)$ is increasing with respec to $\alpha$. This term can be written as
\begin{align*}
    \varepsilon_p\left(\vc{z}^d(\alpha),\kappa\right)^p&=\frac{\sum \limits_{l=\lceil \kappa d \rceil+1 }^{d}z_l(\alpha)^p}{\sum \limits_{l=1 }^{d}z_l(\alpha)^p}=\frac{\sum \limits_{l=\lceil \kappa d \rceil+1 }^{d}\Gamma_l^{-\frac{p}{\alpha}}}{\sum \limits_{l=1 }^{d} \Gamma_l^{-\frac{p}{\alpha}}}\eqqcolon \frac{u}{v}.
\end{align*}
Taking the derivative with respect to $\alpha$ gives
\begin{align*}
    \frac{ \partial \varepsilon_p\left(\vc{z}^d(\alpha),\kappa\right)^p}{\partial \alpha}=\frac{vu'-v'u}{v^2},
\end{align*}
where
\begin{align*}
    vu'-v'u &=\frac{p}{\alpha^2} \left[  \left(\sum \limits_{l=1 }^{d} \Gamma_l^{-\frac{p}{\alpha}}\right)  \left(\sum \limits_{l=\lceil \kappa d \rceil+1 }^{d} \Gamma_l^{-\frac{p}{\alpha}} \log\left(\Gamma_l\right) \right)- \left(\sum \limits_{l=1 }^{d} \Gamma_l^{-\frac{p}{\alpha}} \log\left(\Gamma_l\right) \right)  \left(\sum \limits_{l=\lceil \kappa d \rceil+1 }^{d} \Gamma_l^{-\frac{p}{\alpha}} \right) \right]\\
    &=\frac{p}{\alpha^2} \sum \limits_{l_1=1 }^{\lceil \kappa d \rceil} \sum \limits_{l_2=\lceil \kappa d \rceil+1 }^{d} \Gamma_{l_1}^{-\frac{p}{\alpha}} \Gamma_{l_2}^{-\frac{p}{\alpha}} \left(\log\left(\Gamma_{l_2}\right) -\log\left(\Gamma_{l_1}\right) \right)\\
    &\stackrel{a.s.}{>}0.
\end{align*}
This shows that $\varepsilon_p\left(\vc{z}^d(\alpha),\kappa\right)$ is almost surely strictly increasing with respect to $\alpha$, and consequently $\kappa_p\left(\vc{z}^d(\alpha),\varepsilon\right)$ is almost surely increasing with respect to $\alpha$.

Since $\varepsilon_p(\vc{w},\kappa)$ is a bounded function and almost surely continuous with respect to $\vc{w}$,  $\mathbb{E}\left[\varepsilon_p\left(\vc{w}_i^{d},\kappa\right)\right]$ converges also to $\mathbb{E}\left[\varepsilon_p\left(\vc{z}_i^{\infty}(\alpha_i),\kappa\right)\right]$, for $i=1,2$. To show \eqref{eq:compresExpEr}, choose $d_0(\delta)$ large enough, such that
\begin{align*}
    \bigg| \mathbb{E}\left[\varepsilon_p\left(\vc{w}_1^{d},\kappa\right)\right]-\mathbb{E}\left[\varepsilon_p\left(\vc{z}_1^{\infty}(\alpha_1),\kappa\right)\right]\bigg| & < \frac{\delta}{4},~~     \bigg| \mathbb{E}\left[\varepsilon_p\left(\vc{w}_2^{d},\kappa\right)\right]-\mathbb{E}\left[\varepsilon_p\left(\vc{z}_2^{\infty}(\alpha_2),\kappa\right)\right]\bigg| < \frac{\delta}{4},\\
    \bigg| \mathbb{E}\left[\varepsilon_p\left(\vc{z}_1^d,\kappa\right)\right]-\mathbb{E}\left[\varepsilon_p\left(\vc{z}_1^{\infty}(\alpha_1),\kappa\right)\right]\bigg| & < \frac{\delta}{4},~~
    \bigg| \mathbb{E}\left[\varepsilon_p\left(\vc{z}_2^d,\kappa\right)\right]-\mathbb{E}\left[\varepsilon_p\left(\vc{z}_2^{\infty}(\alpha_2),\kappa\right)\right]\bigg| \leq  \frac{\delta}{4}.
\end{align*}

Then,
\begin{align*}
    \mathbb{E}\left[\varepsilon_p\left(\vc{w}_2^{d},\kappa\right)\right]-\mathbb{E}\left[\varepsilon_p\left(\vc{w}_1^{d},\kappa\right)\right]&< \delta +\mathbb{E}\left[\varepsilon_p\left(\vc{z}_2^d(\alpha_2),\kappa\right)\right] -\mathbb{E}\left[\varepsilon_p\left(\vc{z}_1^d(\alpha_1),\kappa\right)\right] <\delta.
\end{align*}
Similarly,  \eqref{eq:compresExpRate} can be concluded.
\end{proof}

\section{Proofs of the Technical Lemmas} \label{sec:sup_lemProofs}
In this section, we give proofs of all the unproved lemmas stated in the paper.
\subsection{Proof of Lemma~\ref{lem:prunOutputDiff}}
\begin{proof}
Inequality \eqref{eq:netOutputDiff} can be concluded from part i. of Lemma~\ref{lem:diffBound}, stated in Section~\ref{sec:proofGenPrun}, by letting $\upsilon_l=\varepsilon_l \|\vc{w}_l\|$.  Inequality \eqref{eq:netOutputDiffParticular} can be concluded from \eqref{eq:netOutputDiff} and since when $\|\vc{w}\| \leq R$, then
\begin{align*}
 \prod_{l=1}^{L}  \|\vc{w}_l\| \leq  \left(\frac{R}{\sqrt{L}}\right)^{L}.   
\end{align*}\end{proof}

\subsection{Proof of Lemma~\ref{lem:diffBound}}
\begin{proof}
\begin{itemize}[leftmargin=*,wide]
    \item[i.] Similar to \cite{neyshabur2017pac}, we will show the first inequality by induction. Let $f^l_{\vc{w}}(\vc{x})$ denote the output of the $l$th layer: $f_{\vc{w}}^{1}(\vc{x})=\vc{W}_{1} \vc{x}$ and $f_{\vc{w}}^{l}(\vc{x})=\vc{W}_{l} \phi\left(f_{\vc{w}}^{l-1}(\vc{x})\right)$.  We show that for $i \in \llbracket 1, L\rrbracket$,
    following relations hold:
    \begin{align*}
        \| f^i_{\vc{w}}(\vc{x})-f^i_{\vc{w}'}(\vc{x})\| &\leq B \prod \limits_{l=1}^{i} \left(\|\vc{w}_l\|+\upsilon_l \right)-B\prod \limits_{l=1}^{i} \|\vc{w}_l\|.
    \end{align*}
    The induction base $i=0$ holds trivially. Assume that it holds till layer $i$. We show that it holds for layer $i+1$ as well. Note that with our notations $\vc{w}_l=\textbf{vec}\left(\vc{W}_l\right)$ and consequently  $\|\vc{W}_l\|=\|\vc{w}_l\|$.
    \begin{align*}
         \big\| f^{i+1}_{\vc{w}'}&(\vc{x})-f^{i+1}_{\vc{w}}(\vc{x})\big\|\\
         &=   \left\|\vc{W}'_{i+1} \phi\left(f^{i}_{\vc{w}'}(\vc{x})\right)-\vc{W}_{i+1} \phi\left(f^{i}_{\vc{w}}(\vc{x})\right)\right\|\\
         &=  \big\|\left(\vc{W}_{i+1}+\vc{W}'_{i+1}-\vc{W}_{i+1}\right) \left(\phi\left(f^{i}_{\vc{w}}(\vc{x})\right)+\phi\left(f^{i}_{\vc{w}'}(\vc{x})\right)-\phi\left(f^{i}_{\vc{w}}(\vc{x})\right)\right)-\vc{W}_{i+1} \phi\left(f^{i}_{\vc{w}}(\vc{x})\right)\big\| \\
         &\leq    \left\|\left(\vc{W}'_{i+1}-\vc{W}_{i+1}\right) \phi\left(f^{i}_{\vc{w}}(\vc{x})\right)\right\|+  \left\|\vc{W}_{i+1} \left(\phi\left(f^{i}_{\vc{w}'}(\vc{x})\right)-\phi\left(f^{i}_{\vc{w}}(\vc{x})\right)\right)\right\|\\
         &\hspace{0.4 cm}+   \left\|\left(\vc{W}'_{i+1}-\vc{W}_{i+1}\right) \left(\phi\left(f^{i}_{\vc{w}'}(\vc{x})\right)-\phi\left(f^{i}_{\vc{w}}(\vc{x})\right)\right)\right\|\\
         & \stackrel{(a)}{\leq} \upsilon_{i+1}B\prod \limits_{l=1}^i \|\vc{w}_l\|+ \left(\|\vc{w}_{i+1}\|+ \upsilon_{i+1}\right) \left(B \prod \limits_{l=1}^i \left(\|\vc{w}_l\|+\upsilon_l \right)-B\prod \limits_{l=1}^i \|\vc{w}_l\|\right)\\
         &=B \prod \limits_{l=1}^{i+1} \left(\|\vc{w}_l\|+\upsilon_l \right)-B\prod \limits_{l=1}^{i+1} \|\vc{w}_l\|.
    \end{align*}
    where $(a)$ is concluded since  $\phi$ is $1$-Lipschitz, $\phi(0) = 0$, and since due to the structure of $f_{\vc{w}}$, $ \| f^{i}_{\vc{w}}(\vc{x}) \| $ can be upper bounded as
\begin{align*}
    \| f^{i}_{\vc{w}}(\vc{x}) \| \leq \|\vc{x}\| \prod_{l=1}^i \|\vc{w}_i\|\leq B \prod_{l=1}^i \|\vc{w}_i\|.
\end{align*}
    Next, we show the second inequality.
    \begin{align*}
        \big| \ell_{0,\tau}(\vc{z}, f_{\vc{w}})-\ell_{0,\tau}(\vc{z}, f_{\vc{w}'}) \big| &\leq \frac{1}{\tau}\left|f_{\vc{w}}(\vc{x})[y]-\max\limits_{j \neq y}f_{\vc{w}}(\vc{x})[j]-f_{\vc{w}'}(\vc{x})[y]+\max\limits_{j' \neq y}f_{\vc{w}'}(\vc{x})[j']\right|\\
        &\leq \frac{1}{\tau}\left|f_{\vc{w}}(\vc{x})[y]-f_{\vc{w}'}(\vc{x})[y]\right|+\frac{1}{\tau}\left|\max\limits_{j \neq y}f_{\vc{w}}(\vc{x})[j]-\max\limits_{j' \neq y}f_{\vc{w}'}(\vc{x})[j']\right|\\
        &\leq \frac{1}{\tau}\left|f_{\vc{w}}(\vc{x})[y]-f_{\vc{w}'}(\vc{x})[y]\right|+\frac{1}{\tau} \max \limits_{j \neq y} \left|f_{\vc{w}}(\vc{x})[j]-f_{\vc{w}'}(\vc{x})[j]\right|\\
        &\stackrel{(a)}{\leq} \frac{\sqrt{2}}{\tau} \| f_{\vc{w}}(\vc{x})-f_{\vc{w}'}(\vc{x})\|,
        \end{align*}
        where $(a)$ is derived using the relation $x+y\leq \sqrt{2(x^2+y^2)}$, for $x,y \in \mathbb{R}_{+}$.
        \item[ii.] To show the first inequality, note that due to symmetry, R.H.S. of part i. is maximized when $\|\vc{w}_l\|=R/\sqrt{L}$, for $l \in \llbracket 1,L \rrbracket$. Hence,
        \begin{align}
              \| f_{\vc{w}}(\vc{x})-f_{\vc{w}'}(\vc{x})\| &\leq   B \left(  \left(\frac{R}{\sqrt{L}}+\upsilon\right)^L-\left(\frac{R}{\sqrt{L}}\right)^L \right). \label{eq:diffOut1}
        \end{align}
        Next, we show that if $a\geq b \geq 0$ and $n \in \mathbb{N}$, then
        \begin{align}
            a^n-b^n \leq n(a-b)a^{n-1}. \label{eq:diffOut2} 
        \end{align}
        We show this by induction. It trivially holds for $n=1$. Suppose that it holds till $n \leq i-1$. We show that it holds for $n=i$, as well. 
        \begin{itemize}
            \item if $i$ is even, then
            \begin{align*}
                a^i-b^i =\left(a^{\frac{i}{2}}-b^{\frac{i}{2}}\right)\left(a^{\frac{i}{2}}+b^{\frac{i}{2}}\right)&\stackrel{(a)}{\leq} \frac{i}{2}(a-b)a^{i/2-1}\times 2 a^{i/2}= i(a-b)a^{i-1},
            \end{align*}
            where $(a)$ is derived using the induction assumption.
            \item if $i$ is odd, then
            \begin{align*}
                a^i-b^i &=\left(a-b\right)\sum \limits_{k=0}^{i-1}a^k b^{i-1-k}
                \leq i(a-b)a^{i-1}.
            \end{align*}
        \end{itemize}
        Thus, using \eqref{eq:diffOut1} and \eqref{eq:diffOut2} and since $\upsilon \geq R/\sqrt{L}$, 
        \begin{align*}
              \| f_{\vc{w}}(\vc{x})-f_{\vc{w}'}(\vc{x})\| &\leq   BL   \left(\frac{2R}{\sqrt{L}}\right)^{L-1} \upsilon. 
        \end{align*}
        This completes the proof for the first inequality. Finally, the second inequality trivially follows from the first one and part i.
\end{itemize}
\end{proof}

\subsection{Proof of Lemma~\ref{lem:disPointsComp}}
\begin{proof}
Note that there exists $\binom{d}{\kappa d}$ different ways to choose $\kappa d$ coordinates with zero values. Next, each of the resulting $\kappa d$-dimensional sub-space can be discretized using at most $\left(\frac{3 R(\delta) }{\upsilon}\right)^{\kappa d}$ number of points due to \cite[Theorem~14.2.]{Wu2020}. Using  the  following lemma completes the proof.

\begin{lem}[{\cite[Exercise~5.8.b.]{gallager1968}}]  \label{lem:binomBounds} For $n,m \in \mathbb{Z}^+$ and $m \leq n$, $\binom{n}{m} \leq e^{nh_b(m/n)}$.
\end{lem}
\end{proof}

\subsection{Proof of Lemma~\ref{lem:normBoundStable}}
\begin{proof} First we show that in general when random variables $y_i$, $i\in \llbracket 1,m \rrbracket$ are independent and $\sum_{i=1}^m a_i \leq a$, then
\begin{align*}
     \mathbb{P}\left(\sum_{i=1}^m y_i\geq a\right) \leq \sum \limits_{i=1}^m \mathbb{P}\left(y_i \geq a_i\right).
\end{align*}
We prove this for the case of $m=2$, and the general case follows by an induction.
\begin{align*}
     \mathbb{P}(y_1+y_2 \geq a)&=\mathbb{P}(y_1+y_2 \geq a,y_1\geq a_1)+\mathbb{P}(y_1+y_2 \geq a,y_1 < a_1)\\
     &\leq \mathbb{P}(y_1\geq a_1)+\mathbb{P}(y_1+y_2 \geq a,y_1 < a_1)\\
     &\leq \mathbb{P}(y_1\geq a_1)+\mathbb{P}(y_2 \geq a-a_1,y_1 < a_1)\\
     &\leq \mathbb{P}(y_1\geq a_1)+\mathbb{P}(y_2 \geq a_2).
\end{align*}
Next, since stable distributions are continuous distributions, hence $P\left(\|\vc{x}\|\geq R(\delta)\right)=\delta$.

Now, to show the idea, first show that if $\vc{x}$ is an i.i.d. $d$-dimensional vector with $\sas_n(\sigma)$ distributions and $\alpha \in (1,2)$, then for $\delta <  2 d (2-\alpha)^{\alpha}$, 
$ R(\delta)$ can be bounded as
\begin{align}
R(\delta)\leq 3\sigma \sqrt{d}  \left(\frac{4 d }{\delta}\right)^{1/\alpha}. \label{eq:stableNormBoundIID}
\end{align}
To show this,
\begin{align}
    \delta= P\left(\|\vc{x}\|^2\geq R^2(\delta)\right) &\leq \sum \limits_{i=1}^d P\left(\|\vc{x}_i\|^2\geq R^2(\delta)/d\right) \nonumber \\
    &= \sum \limits_{i=1}^d P\left(\|\vc{x}_i\| \geq R(\delta)/\sqrt{d}\right)\stackrel{(a)}{\leq}4d\left(\frac{3\sigma \sqrt{d}}{R(\delta)}\right)^{\alpha}, \label{eq:stableNormBound1}
\end{align}
where $(a)$ holds when $R(\delta)\geq 4\sigma \sqrt{d} /(2-\alpha) $ due to the following inequality from \cite[Theorem~19]{bednorz2020tails}. The result is stated for a $\sas(\sigma_{\alpha})\equiv \sas_n(1)$ distribution, where
\begin{align*}
\sigma_{\alpha}\coloneqq \left(2\Gamma(-\alpha)\cos((2-\alpha)\pi/2)\right)^{1/\alpha}.    
\end{align*}
Here, we state the result for arbitrary $\sas_n(\sigma)$. If $y \sim \sas_n(\sigma)$ and $\alpha \in (1,2)$, then for $a \geq 4 \sigma /(2-\alpha)$
\begin{align*}
    \mathbb{P}\left(y \geq a \right) \leq \frac{16}{3}\left(\frac{2\sigma}{a}\right)^{\alpha}\leq 4\left(\frac{3\sigma}{a}\right)^{\alpha}.
\end{align*}

Re-arranging \eqref{eq:stableNormBound1} and considering the condition $R(\delta)\geq 4\sigma \sqrt{d} /(2-\alpha) $, yields 
\begin{align*}
    R(\delta) \leq \max \left(3\sigma \sqrt{d}  \left(\frac{4 d }{\delta}\right)^{1/\alpha}, \frac{4\sigma \sqrt{d}}{(2-\alpha)} \right).
\end{align*}
Hence, \eqref{eq:stableNormBoundIID} holds, at least when 
\begin{align*}
    3\sigma \sqrt{d}  \left(\frac{4 d }{\delta}\right)^{1/\alpha}\geq \frac{4\sigma \sqrt{d}}{(2-\alpha)},
\end{align*}
which is satisfied when $\delta < 2 d (2-\alpha)^{\alpha}$.

Now, to show the lemma, let $a_l\coloneqq d \sigma^2/\sigma^2_l$. Then, similar steps concludes
\begin{align}
    \delta \leq \sum \limits_{l=1}^L \sum \limits_{i=1}^{d_l} \mathbb{P}\left(\|x_{l,i}\| \geq \frac{R(\delta)}{\sqrt{a_l}}\right) \stackrel{(a)}{\leq} 4 \sum \limits_{l=1}^L \sum \limits_{i=1}^{d_l} \left(\frac{3\sigma \sqrt{d}}{R(\delta)}\right)^{\alpha_l} \stackrel{(b)}{\leq} 4d \left(\frac{3\sigma\sqrt{d}}{R(\delta)}\right)^{\alpha},
\end{align}
where $(a)$ holds when $R(\delta) \geq \max_{l} 4 \sigma \sqrt{d}/(2-\alpha_l)$ and $(b)$ holds when $R(\delta) \geq  3 \sigma \sqrt{d}$. Note that $3 \sigma \sqrt{d} \leq  4 \sigma \sqrt{d}/(2-\alpha_l)$.  Finally, similarly, \eqref{eq:stableNormBoundIID} holds if $\delta <2 d (2-\max_l \alpha_{l})^{\alpha}$.

\end{proof}

\subsection{Proof of Lemma~\ref{lem:disPointsOriginal}}
\begin{proof}
To upper bound $\mathcal{N}'_{\upsilon}$, first consider the space $\mathcal{W}''$, defined as\footnote{
For a set $\mathcal{A} =\{i_1,\ldots,i_r\} \subseteq \llbracket 1,d \rrbracket$, denote $\vc{x}_{\mathcal{A}}\coloneqq \left(x_{i_1},x_{i_2},\ldots,x_{i_r}\right)$.}
\begin{align*}
    \mathcal{W}''&\coloneqq\bigcup \limits_{\substack{\mathcal{A} \\ |\mathcal{A}|=\kappa d}} \mathcal{W}''_{\mathcal{A}},\\
    \mathcal{W}''_{\mathcal{A}} &\coloneqq \left\{\vc{w} \in \mathbb{R}^d \big| \|\vc{w}_{\mathcal{A}}\|\leq R(\delta), \|\vc{w}_{\mathcal{A}^C}\| \leq \varepsilon R(\delta) \right\}.
\end{align*}
Since each of $ \mathcal{W}''_{\mathcal{A}} $ is a convex space, then if $\upsilon<\varepsilon R(\delta)$, by \cite[Theorem~14.2.]{Wu2020}, it can be discretized with $\upsilon$-precision using at most
\begin{align*}
    \left(\frac{3}{\upsilon}\right)^{d}\, \frac{\text{Vol}\left(\mathcal{W}''_{\mathcal{A}}\right)}{\text{Vol}\left(\mathbb{B}_d\right)}&= \left(\frac{3}{\upsilon}\right)^{d} \frac{\frac{\pi^{\kappa d/2}R^{\kappa d}}{\Gamma(\kappa d/2+1)}\times \frac{\pi^{(1-\kappa) d/2} \left(\varepsilon R\right)^{(1-\kappa) d}}{\Gamma((1-\kappa)d/2+1)}}{\frac{\pi^{d/2}R^d}{\Gamma(d/2+1)}}\\
    &= \left(\frac{3 \varepsilon^{(1-\kappa) } R }{\upsilon}\right)^{d} \frac{\Gamma(d/2+1)}{\Gamma(\kappa d/2+1)\Gamma((1-\kappa)d/2+1)},
\end{align*}
number of points, where $\mathbb{B}_d$ is the $d$-dimensional unit ball. Now, since $\mathcal{W}(R(\delta),\varepsilon,d,\kappa) \subseteq \mathcal{W}'' $,
\begin{align}
    \mathcal{N}'_{\upsilon} & \leq \binom{d}{\kappa d}  \left(\frac{3 \varepsilon^{(1-\kappa)} R }{\upsilon}\right)^{d} \frac{\Gamma(d/2+1)}{\Gamma(\kappa d/2+1)\Gamma((1-\kappa)d/2+1)}, \nonumber\\
     & \stackrel{(a)}{\leq} e^{d\left(h_b(\kappa)+h_b^{(1)}(\kappa)\right)} \left(\frac{3 \varepsilon^{(1-\kappa)} R }{\upsilon}\right)^{d} \nonumber \\
     &\stackrel{(b)}{=} \left(\frac{3 \rho_{\varepsilon}(\kappa,d) R }{\upsilon}\right)^{d},      \label{eq:cardinality1}
\end{align}
where $(a)$ is concluded from Lemma~\ref{lem:binomBounds} and the following lemma and $(b)$ is concluded since one way to discretize $\mathcal{W}(R(\delta),\varepsilon,d,\kappa) $ is to consider the whole sphere with radius $R$, which needs at most $(3R/\upsilon)^d$, due to \cite[Theorem~14.2.]{Wu2020}.

\begin{lem} \label{lem:binomBounds2} For $n,m \in \mathbb{Z}^+$, $n \geq m$, and $n \geq 10$,
\begin{align*}
    \frac{\Gamma(n/2+1)}{\Gamma(m/2+1)\Gamma((n-m)/2+1)} \leq e^{\lceil n/2\rceil \max\left(h_b(\lceil m/2 \rceil/\lceil n/2\rceil),h_b(\lfloor m/2\rfloor/\lceil n/2\rceil)\right)}.
\end{align*}
\end{lem}
\end{proof}

\subsection{Proof of Lemma~\ref{lem:binomBounds2}}
\begin{proof}
For $m=0$ or $m=n$, the claim holds with equality. Let $1 \leq m \leq n-1$. When $n$ and $m$ are even, then the lemma can be concluded from Lemma~\ref{lem:binomBounds}. Assume, at least one of $n$ and $m$ are odd numbers. We consider two cases of $n$ being odd and even separately.

Note that for $a \in \mathbb{N}$, \cite{Robbins1955}
\begin{align}
    \Gamma(a+1) = \sqrt{2 \pi a}\, a^a e^{-a} e^{r_a}, \label{eq:gamma1}
\end{align}
where $1/(12a+1)<r_a < 1/(12a)$. Moreover,
\begin{align}
    \Gamma\left(a+\frac{1}{2}\right) = \frac{\sqrt{\pi} (2a)!}{4^a a!}   &\stackrel{(*)}{=} \sqrt{2 \pi}\, a^a e^{-a} e^{s_a}, \label{eq:gamma2}
\end{align}
where $(*)$ is derived using \eqref{eq:gamma1} with $s_a$ being bounded as 
\begin{align*}
    \frac{1}{24a+1}-\frac{1}{12a} < s_a < \frac{1}{24a}-\frac{1}{12a+1}.
\end{align*}

\begin{itemize}
    \item[Odd $n$:]     Let $n=2k+1$ and $m=2q$, where $1<q\leq k$.  Then,
    \begin{align}
        \frac{\Gamma(n/2+1)}{\Gamma(m/2+1)\Gamma((n-m)/2+1)} & \stackrel{(a)}{=} \frac{e^{s_{k+1}-\left(r_{q}+s_{k+1-q}\right)} (k+1)^{k+1}} {\sqrt{2 \pi q} \, q^{q}(k+1-q)^{k+1-q}} \nonumber \\
        &  \stackrel{(b)}{\leq} \frac{ (k+1)^{k+1}} {q^{q}(k+1-q)^{k+1-q}} \nonumber \\
        &<e^{(k+1)h_b\left(q/(k+1)\right)} \nonumber\\
        &<e^{\lceil n/2\rceil \max\left(h_b( m/2\lceil n/2\rceil),h_b(m/2\lceil n/2\rceil)\right)}, \label{eq:gamma4} 
    \end{align}
    where $(a)$ is derived using \eqref{eq:gamma1} and \eqref{eq:gamma2}, and $(b)$ is derived, since
     \begin{align*}
       s_{k+1}{-}\big(r_{q}+&s_{k+1-q}\big)\\ &<\frac{1}{24k+24}-\frac{1}{12k+13}-\frac{1}{12q+1}-\frac{1}{24(k-q)+25}+\frac{1}{12(k-q)+12}\\
       &\leq\frac{1}{24k+24}-\frac{1}{12k+13}-\frac{1}{12k+1}-\frac{1}{25}+\frac{1}{12}\\
       &<\frac{1}{12}-\frac{1}{25}\leq 0.05.
    \end{align*}
    
    The case of $m$ being odd is similar.

    \item[Even $n$:]  
    Let $n=2k$ and $m=2q+1$, where $1<q<k$ and $k\geq 5$. Then,
    \begin{align}
        \frac{\Gamma(n/2+1)}{\Gamma(m/2+1)\Gamma((n-m)/2+1)} & \stackrel{(a)}{=} \frac{e^{1+r_k-\left(s_{q+1}+s_{k-q}\right)}\sqrt{ k} \, k^k} {\sqrt{2\pi}(q+1)^{q+1}(k-q)^{k-q}} \nonumber \\
        &\stackrel{(b)}{<} \frac{1.17 \sqrt{k} \, k^k} {(q+1)^{q+1}(k-q)^{k-q}} \label{eq:gamma3} 
    \end{align}
    where $(a)$ is derived using \eqref{eq:gamma1} and \eqref{eq:gamma2} and $(b)$ is derived since
    \begin{align*}
        r_k-\left(s_{q+1}+s_{k-q}\right) &< \frac{1}{12k}-\frac{1}{24q+25}+\frac{1}{12q+12}-\frac{1}{24(k-q)+1}+\frac{1}{12(k-q)}\\
       &\leq \frac{1}{6k}-\frac{1}{24k+1}-\frac{1}{25}+\frac{1}{12}\\
       &<0.07,
    \end{align*}
    where the last step holds for $k\geq 5$. 
    
    Next, for $k \geq 5$, either $q+1 \geq 1.17 \sqrt{k}$ or $k-q \geq 1.17 \sqrt{k}$. Otherwise, we would conclude $k+1<2.34 \sqrt{k}$, which is a contradiction for $k \geq 5$.
    \begin{itemize}
        \item If $q+1 \geq 1.17\sqrt{k}$, then \eqref{eq:gamma3} is upper bounded by
        \begin{align*}
            \frac{1.17\sqrt{k} \, k^k} {(q+1)^{q+1}(k-q)^{k-q}}& \leq \frac{ k^k} {(q+1)^{q}(k-q)^{k-q}}
            \\
            &\leq \frac{k^k} {q^{q}(k-q)^{k-q}} \\
            &\leq e^{kh_b(q/k)}=e^{\frac{n}{2}h_b\left(2\lfloor m/2 \rfloor/n\right)}.
        \end{align*}
        \item If $k-q \geq 1.17 \sqrt{k}$, then \eqref{eq:gamma3} is upper bounded by
        \begin{align*}
            \frac{1.17 \sqrt{k} \, k^k} {(q+1)^{q+1}(k-q)^{k-q}}& \leq \frac{ k^k} {(q+1)^{q+1}(k-q)^{k-q-1}}
            \\
            &\leq \frac{k^k} {(q+1)^{q+1}(k-q-1)^{k-q-1}} \\
            &\leq e^{kh_b((q+1)/k)}=e^{\frac{n}{2}h_b\left(2\lceil m/2 \rceil/n\right)}.
        \end{align*}
    \end{itemize}
\end{itemize}
\end{proof}

\end{appendices}
\end{document}